\documentclass{article}

\usepackage[utf8]{inputenc}
\usepackage[T1]{fontenc}
\usepackage{authblk}
\usepackage{fullpage}
\usepackage{amsmath}
\usepackage{amsthm}
\usepackage{xspace}
\usepackage{thmtools,thm-restate}
\usepackage{subcaption}
\usepackage{mkolar_definitions}
\usepackage{graphicx,scalerel}
\usepackage[square]{natbib}
\usepackage{dsfont}
\usepackage[Smaller]{cancel}
\usepackage{setspace} \setdisplayskipstretch{0.5}
\usepackage{enumerate}
\usepackage{enumitem}
\usepackage[normalem]{ulem}
\usepackage[vlined,ruled,linesnumbered,noend]{algorithm2e}
\usepackage[most]{tcolorbox}
\usepackage{booktabs}
\usepackage{titletoc}

\usepackage{natbib}
\setcitestyle{square}
\usepackage{bibunits}
\defaultbibliography{ref}
\defaultbibliographystyle{apalike}

\usepackage{nameref}
\usepackage[hypertexnames=false]{hyperref}
\usepackage{cleveref}

\usepackage{xcolor}
\definecolor{color1}{HTML}{105e8a}
\definecolor{color2}{HTML}{e99926}
\definecolor{color3}{HTML}{b82a0c}
\definecolor{color4}{HTML}{3e8a10}
\definecolor{color5}{HTML}{80037e}

\hypersetup{
    hidelinks,
    colorlinks,
    linkcolor=color1,
    citecolor=color1,
    urlcolor=color1
}

\newcommand{\linearspan}{\mathrm{span}}
\newcommand{\sign}{\mathrm{sign}}

\newcommand{\indep}{\texttt{independent-ERM}\xspace}
\newcommand{\oracle}{\texttt{oracle}\xspace}

\usepackage{apptools}
\AtAppendix{\counterwithin{theorem}{section}}
\AtAppendix{\counterwithin{remark}{section}}
\AtAppendix{\counterwithin{figure}{section}}
\AtAppendix{\counterwithin{definition}{section}}
\AtAppendix{\counterwithin{lemma}{section}}
\AtAppendix{\counterwithin{proposition}{section}}
\AtAppendix{\counterwithin{definition}{section}}
\AtAppendix{\counterwithin{corollary}{section}}
\AtAppendix{\counterwithin{example}{section}}
\AtAppendix{\counterwithin{claim}{section}}
\AtAppendix{\counterwithin{fact}{section}}

\newenvironment{restate}[2]{%
  \par\noindent\textbf{#1~\ref{#2}.}\ \itshape
}{\par\normalfont}

\title{Bridging Lifelong and Multi-Task Representation Learning via Algorithm and Complexity Measure} 

\author[1]{Zhi Wang}
\author[2]{Chicheng Zhang}
\author[1]{Ramya Korlakai Vinayak}
\affil[1]{University of Wisconsin--Madison}
\affil[2]{University of Arizona}
\date{}
\begin{document}

\maketitle

\begin{abstract}%
In lifelong learning, a learner faces a sequence of tasks with shared structure and aims to identify and leverage it to accelerate learning. We study the setting where such structure is captured by a common representation of data. Unlike multi-task learning or learning-to-learn, where tasks are available upfront to learn the representation, lifelong learning requires the learner to make use of its existing knowledge while continually gathering partial information in an {\em online} fashion. In this paper, we consider a generalized framework of lifelong representation learning. We propose a simple algorithm that uses multi-task empirical risk minimization as a subroutine and establish a sample complexity bound based on a new notion we introduce---the {\em task-eluder dimension}. Our result applies to a wide range of learning problems involving general function classes. As concrete examples, we instantiate our result on classification and regression tasks under noise.
\end{abstract}

\begin{bibunit}

\section{Introduction}
\label{sec:introduction}

In many real-world settings, learning naturally involves a collection of related tasks~\citep{caruana1997multitask}. The ability to identify and leverage shared structure among tasks allows a learner to transfer knowledge and accelerate learning. One common form of structure lies in a shared {\em representation} of data such that simple functions operating on it can support effective and efficient learning across tasks. For example, adapting linear classifiers over a pre-trained set of deep neural network features has demonstrated state-of-the-art performance in computer vision \citep{donahue2014decaf}.

The benefit of representation transfer is a central topic in the study of multi-task learning (MTL) and learning to learn (LTL) \citep{baxter2000model,maurer2016benefit,tripuraneni2020theory,aliakbarpour2024metalearning}. In MTL, the learner is given a fixed set of tasks and aims to jointly learn a shared representation and task-specific prediction layers. In LTL, also known as meta-learning, tasks are drawn from an unknown distribution: the learner is first trained on a collection of tasks and then evaluated on a new task sampled from the same distribution. In both cases, tasks are available upfront, and LTL often relies on having enough {\em diversity} in seen tasks for the learner to fully identify the representation before applying it in an unseen task \citep{tripuraneni2020theory,du2021fewshot}.

However, learning may unfold over time, with tasks arriving in a sequence, as is the case in how humans learn. On the one hand, the learner {\em should} be able to begin making use of the shared structure before it is fully uncovered. On the other hand, the learner can always continually gather partial information about the structure and refine its internal representation. This motivates the study of {\em lifelong representation learning}, where the learner aims to identify and utilize a common representation in an online manner to reduce the sample complexity over the sequence of tasks.

\begin{figure*}[t]
    \begin{subfigure}{0.45\textwidth}
        \centering
        \includegraphics[width=\linewidth]{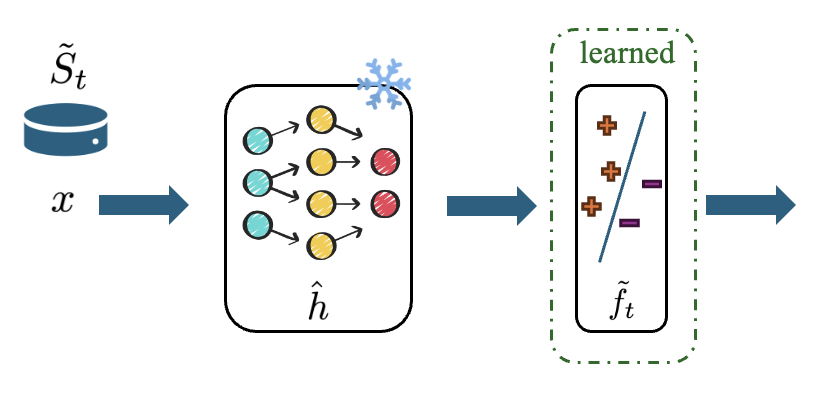}
        \caption{Few-shot property test}
        \label{fig:few-shot}
    \end{subfigure}
    \hfill
    \begin{subfigure}{0.47\textwidth}
        \includegraphics[width=\linewidth]{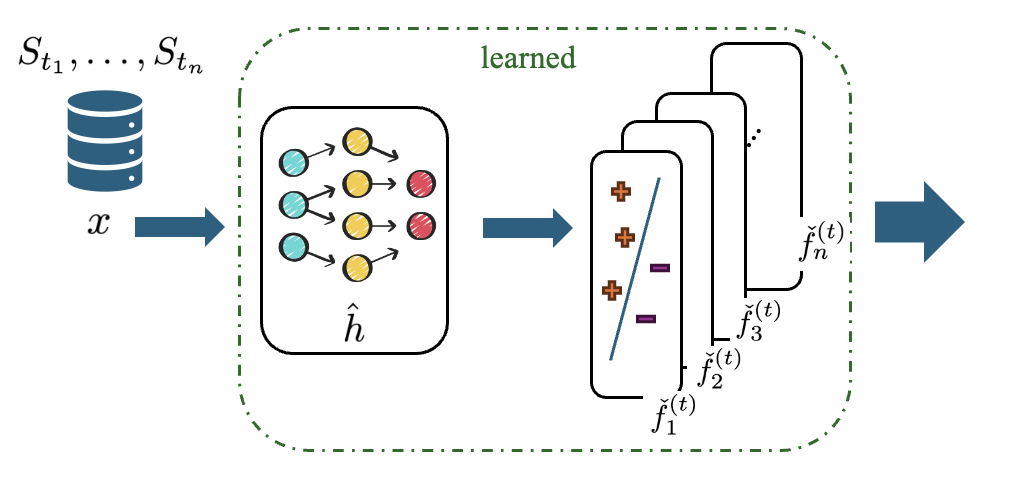}
        \caption{Multi-task ERM}
        \label{fig:mtl}
    \end{subfigure}
    \caption{Our algorithm maintains a representation $\hat{h}$. When a new task arrives, the algorithm first performs a few-shot property test to check whether $\hat{h}$ admits a prediction layer with low excess risk. If not, it performs MTL on data from a subset of previously seen tasks and updates $\hat{h}$.}
    \label{fig:illustration}
\end{figure*}

While MTL and LTL with a shared representation have been extensively studied, theoretical understanding of the lifelong/online regime remains relatively limited (see Section~\ref{sec:related-work} for related work). \citet{balcan2015efficient} study lifelong learning of linear classifiers that lie in a low-dimensional linear subspace in the noiseless, realizable setting. More recently, \citet{cao2022provable} consider a broader family of representations---namely, one hidden-layer neural networks---again under the noiseless, realizable assumption. In this work, our goal is to develop an algorithm and a theory for general function classes in noisy but well-specified settings (Section~\ref{sec:preliminaries}).
Our main contributions are:
\begin{enumerate}[leftmargin=*]
    \item We study a simple, provably efficient algorithm (Section~\ref{sec:algorithm}) for lifelong representation learning with composite predictors $f_t \circ h$, where $h \in \Hcal$ is a shared representation (e.g., a deep neural network) and $f_t \in \Fcal$ is a prediction layer specialized to task $t$ (e.g., a linear function). 
    Our algorithm is composed of two modular subroutines,    
    few-shot property test and multi-task empirical risk minimization (ERM).
    This design points toward a connection between the theory of lifelong learning and more practical algorithmic frameworks in real-world MTL.
    
    \item We provide theoretical guarantees on the sample and space complexities of our algorithm (Section~\ref{sec:main-result}) based on a new complexity measure, the {\em task-eluder dimension} (Section~\ref{sec:task-eluder}).
    Together, our algorithm and complexity measure bridge lifelong learning and MTL: multi-task ERM serves as a mechanism for refining the learner's representation in lifelong learning, and the task-eluder dimension bounds how many times it needs to be invoked.
    When the task-eluder dimension is small, the benefit of representation transfer in lifelong learning becomes more evident.
    
    \item To demonstrate the applicability of our findings, we provide examples of how our result can be instantiated in noisy regression and classification tasks (Section~\ref{sec:examples}), and we empirically validate our theoretical results on synthetic and semi-synthetic data (Section~\ref{sec:empirical-validation}).
\end{enumerate}

\section{Preliminaries}
\label{sec:preliminaries}

\paragraph{Notation.} Throughout, we denote by $[n] := \cbr{1, \ldots, n}$. We use $(f \circ h)(x) = f(h(x))$ to denote the composition of functions $f$ and $h$. We write $\lesssim$ and $\gtrsim$ for inequality up to a universal constant factor. For matrices $A$ and $B$, we use $A \precsim B$ to denote that there exists a universal constant $c > 0$ such that $A \preceq cB$. $\tilde{\Ocal}$ hides logarithmic factors.

\subsection{Problem formulation}
\label{sec:problem-formulation}

We now present the lifelong representation learning problem studied in this work, which generalizes the settings in \citep{balcan2015efficient,cao2022provable}. 
Consider a sequence of $T$ tasks arriving over time. They share common input and output spaces, denoted by $\Xcal \subset \RR^d$ and $\Ycal \subset \RR$.
Each task $t \in [T]$ is associated with an underlying data-generating distribution $\Pcal_t$ over $\Xcal \times \Ycal$.\footnote{Following prior work, we assume that when a task is completed, it does not come back. For example, even when task $10$ has the same distribution as task $1$, we treat it as a new task and may produce a different model. Extending our setting to continual learning~\citep{kirkpatrick2017overcoming} where old tasks can reappear is an interesting future direction.
}
We assume that $\Pcal_1, \ldots, \Pcal_T$ share a common marginal distribution\footnote{With some care, our results can be extended to handle heterogeneous marginal distributions that are well-conditioned.
} over $\Xcal$, denoted by $P_X$, as is standard in the literature for studying representation transfer \citep{tripuraneni2020theory,xu2021representation}.

Let $\Hcal \subseteq \cbr{h: \Xcal \rightarrow \Zcal}$ be a class of representations and $\Fcal \subseteq \cbr{f: \Zcal \rightarrow \RR}$ be a class of task-specific prediction layers that operate on these representations. A predictor for a task is then given by $f \circ h$. As an example, $\Hcal$ may represent a family of deep neural networks that map $\Xcal$ to $\Zcal = \RR^k$, while $\Fcal$ is a class of linear functions in $\RR^k$. Let $\Fcal \circ \Hcal := \cbr{f \circ h: f \in \Fcal, h \in \Hcal}$; we assume that $\Fcal \circ \Hcal$ is permissible (see \citep{pollard1984convergence} and \citep[Appendix D therein]{baxter2000model}).

Let $\ell: \RR \times \Ycal \rightarrow [0,1]$ be a loss function. For any distribution $P$ over $\Xcal \times \Ycal$, denote by %
\begin{align*}
    \Lcal_P(f \circ h) := \EE_{(x,y) \sim P} \sbr{\ell \rbr{ (f\circ h)(x), y}}
\end{align*}
the risk of $f \circ h$, for any $f \in \Fcal$ and $h \in \Hcal$.

In this paper, we consider a family of distributions parameterized by $g: \Xcal \rightarrow \RR$ of the form
\[
P_g(x,y) = P_X(x) P_{Y|X}(y | x; g).
\]
We assume that the loss $\ell$ is calibrated with respect to $P_g$,
i.e., 
$g$ is Bayes optimal with respect to $P_{g}$ and $\ell$---formally, 
\[
g = \argmin_{g': \Xcal \to \RR } \Lcal_{P_g}(g').
\]
The specific form of $P_{Y | X} (y | x; g)$ may depend on the exact learning problem at hand. For example, in classification under realizability, labels are given deterministically as $y = g(x)$. 
In regression where the response variable follows a distribution from an exponential family, such as logistic regression or Poisson regression,
$P_{Y \mid X}(y \mid x; g) = \lambda(y) \exp(g(x) y - b(g(x)))$,
where $\lambda$ is a base measure and $b(\theta) = \ln \int_{\Ycal} e^{\theta y} \lambda(dy)$ is the associated log partition function.
We refer to $\PP := \cbr{P_{f \circ h}: f \in \Fcal, h \in \Hcal}$ as the probabilistic model of the problem. \\

We assume that the tasks admit a {\em shared representation} under which they are well-specified.
\begin{assumption}[Well-specified model]
\label{assum:well-specified}
There exist $h^* \in \Hcal$ and $f_1^*, \ldots, f_T^* \in \Fcal$ such that, for each task $t \in [T]$, $\Pcal_t = P_{f^*_t \circ h^*}$.
\end{assumption}

\paragraph{Lifelong learning.} 
For each task $t \in [T]$, the learner can request i.i.d.\ samples from $\Pcal_t$, possibly over more than one rounds. Using these samples, the learner produces a predictor $f_t \circ h_t$. Once completed, the learner proceeds to the next task and may never revisit a previous one. However, we assume that the learner may have access to a memory buffer of $o(T)$ size to store some data. Let $\delta, \epsilon \in (0,1)$. The learner's goal is to return predictors $f_t \circ h_t$'s such that
\[
\Pr \big(\forall t \in [T], \ \ \underbrace{\Lcal_{\Pcal_t} (f_t \circ h_t) - \Lcal_{\Pcal_t} (f_t^* \circ h^*)}_{\text{excess risk for task $t$}} \le \epsilon \big) \ge 1 - \delta.
\]
The performance of the learner is measured by its {\em sample complexity}, i.e., the total number of samples it acquires across all $T$ tasks.
We note that the lifelong learning objective here is akin to that of the KWIK (knows what it knows) framework studied in reinforcement learning and active learning \citep{li2008knows}. Here, for each task, the learner can always choose to acquire more data (equivalent to saying ``I don't know'') if they are uncertain about their current prediction. Since the task sequence may be chosen {\em adversarially} (so long as Assumption~\ref{assum:well-specified} is satisfied), the total number of samples naturally accounts for variation in task difficulty, and an important challenge in lifelong representation learning lies in carefully managing when (or in which tasks) to request more data. %

\begin{remark}[Comparison with prior work]
    \label{rem:comparison-settings}
    Our framework generalizes the settings in \citep{balcan2015efficient,cao2022provable} as follows. Earlier formulations assume specific function classes (e.g., $\Fcal$ as a class of linear functions, and $\Hcal$ as a class of low-dimensional linear representations or one-hidden-layer neural networks), whereas our formulation accomodates {\em general function classes}. %
    
    \citet{balcan2015efficient} and \citet{cao2022provable} assume noiseless realizability, i.e., $\Lcal_{\Pcal_t}(f_t^* \circ h^*) = 0$.
    In contrast, we consider a well-specified model in which $f_i^* \circ h^*$ may have non-zero risk due to noise (Assumption~\ref{assum:well-specified}); we defer the extension to the full agnostic setting to future~work.
    
    Lastly, while prior work focuses on specific learning problems (e.g., binary classification in \citep{balcan2015efficient}), our formulation---through suitable choices of $\Xcal$, $\Ycal$, and $\ell$---offers a unified framework for studying lifelong representation learning across a wide range of learning problems.
\end{remark}

\subsection{Background: empirical risk and multi-task ERM}

Before presenting our algorithm and main results, we provide background on multi-task ERM which we use as a subroutine.
Let $P$ be the data-generating distribution over $\Xcal \times \Ycal$ for a task, and $S = \cbr{(x_j, y_j)}_{j=1}^m$ be drawn i.i.d.\ from $P^m$. For any $h \in \Hcal$ and $f \in \Fcal$, define
\begin{align*}
    \widehat{\Lcal}_{S} (f \circ h) := \frac{1}{m} \sum_{j=1}^m \ell \rbr{(f \circ h)(x_j), y_j}
\end{align*}
to be the empirical risk of $f \circ h$ over $S$.

Multi-task ERM for representation transfer has been widely studied in the literature \citep[e.g.,][]{baxter2000model,maurer2016benefit,tripuraneni2020theory,aliakbarpour2024metalearning}. 
Let $P_1, \ldots, P_n$ be the data-generating distribution of $n$ tasks. Suppose for each task $i \in [n]$, we draw an i.i.d.\ sample of size $m$, $S_i \sim P_i^m$, then multi-task ERM jointly finds a shared representation and task-specific prediction layers by solving the following optimization problem:
\begin{align}
    \label{eqn:multi-task-erm}
     \argmin_{\substack{h \in \Hcal \\ f_1, \ldots, f_n \in \Fcal}} \frac{1}{n} \sum_{i=1}^n \widehat{\Lcal}_{S_i} (f_i \circ h).
\end{align}
When $n=1$, we refer to this approach as single-task ERM, or simply ERM.

\section{Algorithm and complexity measure for lifelong representation learning}
\label{sec:algorithm}

\begin{algorithm}[t]
\SetNoFillComment
\SetCommentSty{commentfont}
\caption{Lifelong representation learning with multi-task ERM as a subroutine}
\label{alg:lifelong-nonparametric}
\KwIn{$\Hcal$, $\Fcal$, target error $\epsilon$, confidence $\delta$, number of tasks $T$, noise levels $(\kappa_t)_t$\;}

\vspace{2pt}
Initialize memory $\Mcal \gets \emptyset$\;

\vspace{2pt}
{Initialize $N \gets 1$ (unless $\dim(\Hcal, \Fcal, \epsilon)$ is known (see Definition~\ref{def:task-eluder-dimension}), in which case set $N$ to it)\;} \label{line:initialize-N}

\vspace{2pt}
\For{task $t = 1$}{
    \vspace{2pt}
    Draw a sample $S_1$ of size $m_N$ from $\Pcal_1^{m_N}$, apply ERM to learn $\hat{h}$ and $\hat{f}_1$\, and output $\hat{f}_1 \circ \hat{h}$\;

    \vspace{2pt}
    Set $n \gets 1$ and $t_n \gets t$, and update the memory $\Mcal \gets \Mcal \cup \cbr{S_{t_n}}$\;
}

\vspace{5pt}
\For{tasks $t = 2, \ldots, T$}{
    \vspace{2pt}
    \tcp{Few-shot property test: check if $\hat{h}$ admits a hypothesis for current task $t$ with risk at most $\epsilon$} 
    Draw a sample $\tilde{S}_t$ of size $\tilde{m}$ from $\Pcal_t^{\tilde{m}}$, and apply ERM with current $\hat{h}$ to learn $\tilde{f}_t$\; \label{line:property-test}

    \vspace{2pt}
    \If{$\widehat{\Lcal}_{\tilde{S}_t}(\tilde{f}_t \circ \hat{h}) \le \kappa_t + \frac{3}{4}\epsilon$}{ \label{line:test-success-begin} 
        \vspace{3pt}
        Output $\tilde{f}_t \circ \hat{h}$\; \label{line:test-success-end}
    }
    \Else{
        \vspace{2pt}
        \color{color3}
        \If{$n = N$}{ \label{line:double-start}
            \vspace{2pt}
            Set $n \gets 1$, $N \gets 2N$, and clear the memory $\Mcal \gets \emptyset$\; \label{line:double-end}
        }
        \color{black}
        \Else{
            \vspace{2pt}
            Set $n \gets n+1$;
        }
    
        \vspace{5pt}
        Draw a sample $S_t$ of size $m_N$ from $\Pcal_t^{m_N}$\; \label{line:mtl-begin}
        \vspace{2pt}
        Set $t_n \gets t$, and update the memory $\Mcal \gets \Mcal \cup \cbr{S_{t_n}}$\;
        \vspace{2pt}
        Apply ERM over the samples stored in the memory to learn
        \[
        \hat{h}, \check{f}^{(t)}_1, \ldots, \check{f}^{(t)}_{n} \gets \argmin_{\substack{h \in \Hcal \\ f_1, \ldots, f_{n} \in \Fcal}} \frac{1}{n} \sum_{i=1}^{n} \widehat{\Lcal}_{S_{t_i}} (f_i, h); \label{line:multi-task-erm}
        \]

        Set $\hat{f}_t \gets \check{f}^{(t)}_{n}$, update $\hat{h}$, and output $\hat{f}_t \circ \hat{h}$\; \label{line:mtl-end}
    }
}
\end{algorithm}

\subsection{Algorithm}

Our algorithm builds upon the methods of \citet{balcan2015efficient} and \citet{cao2022provable}, which are tailored to specific function classes under the realizability assumption. In contrast, our algorithm is amenable to general function classes and noise models. In addition, our algorithm is designed with practical considerations in mind, leveraging MTL as a subroutine, which has been extensively studied and aligns with real-world implementations. Algorithm~\ref{alg:lifelong-nonparametric} contains the~pseudocode.

\paragraph{High-level overview.} In our algorithm, the learner maintains a representation $\hat{h} \in \Hcal$, which is initially learned from the first task. For each subsequent task, the learner performs a property test to determine if there exists a prediction layer $f \in \Fcal$ such that the predictor $f \circ \hat{h}$ has excess risk at most $\epsilon$. Specifically, given the restricted class of functions, $\{f \circ \hat{h}: f \in \Fcal\} \subset \Fcal \circ \Hcal$, the learner checks whether the best predictor in this restricted class 
has risk $\epsilon$-close to that of the Bayes optimal predictor.
This is done via few-shot learning the prediction layer with $\hat{h}$ frozen (see also Figure~\ref{fig:few-shot}, and requires a small number of samples depending only on the complexity of $\Fcal$. We call this the {\em few-shot property test}. 

If the test succeeds, the learner simply retains $\hat{h}$ and safely outputs the predictor (lines~\ref{line:test-success-begin} to~\ref{line:test-success-end}). Otherwise, the learner updates its maintained representation $\hat{h}$ by performing multi-task ERM on a subset of tasks where few-shot property test previously failed---whose samples are stored in memory (lines~\ref{line:mtl-begin} to~\ref{line:mtl-end}, see also Figure~\ref{fig:mtl}). Applying multi-task ERM ensures that the new $\hat{h}$ is an effective representation for this subset of tasks. We note that only samples from tasks where the property test failed are added to the memory buffer. To facilitate the few-shot property test, we make an additional assumption.
\begin{assumption}[Known noise levels]
\label{assum:noise-levels}
Let $\kappa_t := \Lcal_{\Pcal_t} (f_t^* \circ h^*)$ denote the Bayes-optimal risk for each task $t \in [T]$. The learner knows the values $(\kappa_t)_{t=1}^T$.
\end{assumption}

\begin{remark}
    While not explicitly stated, the noiseless realizability assumption in \citep{balcan2015efficient,cao2022provable} entails that the learner knows $\kappa_t = 0$ for all $t \in [T]$.
    We introduce Assumption~\ref{assum:noise-levels} because the excess risk measured within $\{f \circ \hat{h}: f \in \Fcal\} \subset \Fcal \circ \Hcal$ can be misleading in estimating the excess risk with respect to Bayes optimal predictor, if only a limited number of examples are sampled for task $t$, as the risk of the best predictor in this class may itself be far from $\Lcal_{\Pcal_t}(f_t^* \circ h^*)$.    Proposition~\ref{prop:property-test-neg} highlights the hardness in a simple linear setting; 
    its formal statement and proof (based on a reduction from \citep[Proposition 2]{kong2018estimating}) are deferred to Appendix~\ref{app:hardness-property-test}.
    We conjecture that it may be impossible to design an algorithm whose sample complexity adapts to the quality of representation of $\hat{h}$ without the knowledge of $\kappa_t$.

    That said, Assumption~\ref{assum:noise-levels} can be relaxed to the knowledge of an upper bound $\kappa$ on the Bayes optimal risk for each task. In this case, the objective for each task would then be to learn a hypothesis with risk at most $\kappa + \epsilon$.

\end{remark}

\vspace{-10pt}
\begin{proposition}[informal]
    \label{prop:property-test-neg}
    Suppose we observe $n$ examples $\cbr{(x_i, y_i)}_{i=1}^n \sim P^n$, where $P$ denotes some noisy linear regression model in $\RR^d$. Let $\Gcal$ be a class of linear predictors in $\RR^d$ and $\Gcal_0 \subset \Gcal$ be restricted to a fixed subspace of dimension $r \le \frac{d}{2}$; that is, $g \in \Gcal_0$ uses a given linear representation. For any $P$ and $\Gcal' \subseteq \Gcal$, let $\kappa_{P}(\Gcal') := \inf_{g \in \Gcal'} \EE_{(x,y) \sim P} \sbr{(g(x) - y)^2}$.
    Consider two hypotheses:
    \[
    H_0 = \cbr{P: \kappa_P(\Gcal_0) =\kappa_P(\Gcal)} \quad \text{and} \quad H_1 = \cbr{P: \kappa_P(\Gcal_0) > \kappa_P(\Gcal) + 0.9}.
    \]
    There exists some constant $c$ such that no test can successfully distinguish between $H_0$ and $H_1$ with probability $\frac{2}{3}$ using fewer than $c\sqrt{d}$ samples.

\end{proposition}

\vspace{-6pt}
\begin{remark}
The use of a memory buffer in lifelong learning is not specific to our algorithm \citep[e.g.][]{isele2018selective}. \citet{cao2022provable} utilize past features stored in a memory buffer to perform representation refinement in their main algorithm, \texttt{LLL-RR}, and they also  suggest a heuristic algorithm, \texttt{H-LLL}, that stores and reuses training datasets.
\end{remark}

It remains to understand when property tests in Algorithm~\ref{alg:lifelong-nonparametric} stop failing---that is, when multi-task ERM is no longer needed and how large the memory buffer must be. We now introduce a complexity measure that characterizes the sample and memory requirements of our algorithm, before returning to examine our algorithm in greater detail.

\vspace{-4pt}
\subsection{The task-eluder dimension}
\label{sec:task-eluder}
The eluder dimension is a well-established complexity measure studied in sequential decision-making \citep[e.g.,][]{russo2013eluder,foster2020instance,li2022understanding,hanneke2024star}. Its essence is captured by the canonical illustrative example: how many times can a politician ``elude'' by answering questions without revealing their true position \citep{russo2013eluder}?
In this work, we adapt and extend the eluder dimension to characterize the complexity of lifelong representation learning.
\begin{definition}[$\epsilon$-independence]
\label{def:eps-independence}
Let $\Hcal \subset \cbr{h: \Xcal \rightarrow \Zcal}$ be a class of representations, and $\Fcal \subset \cbr{f: \Zcal \rightarrow \RR}$ be a class of prediction layers that operate on these representations. A predictor is the composition $f \circ h$, where $f \in \Fcal$ and $h \in \Hcal$. Let $\PP = \cbr{P_{f \circ h}: f \in \Fcal, h \in \Hcal}$ be the probabilistic model and $\ell: \RR \times \Ycal \rightarrow [0,1]$ be a loss function. \\

\vspace{-10pt}
\noindent For any representation $h \in \Hcal$, we say $(h, f_{n})$ is {\em $\epsilon$-independent} of $\cbr{(h, f_1), \ldots, (h, f_{n-1})}$ with respect to $(\Hcal, \Fcal)$ if there exist $h' \in \Hcal$ and $f'_1, \ldots, f'_{n-1} \in \Fcal$ such that
\begin{align}
    \sum_{i=1}^{n-1} \EE_{P_{f_i \circ h}} \sbr{ \ell((f'_i \circ h')(x), y) - \ell((f_i \circ h)(x), y) }\le \epsilon, \label{eqn:eps-independence-sum}
\end{align}
but for any $f'_{n} \in \Fcal$, $\EE_{P_{f_n \circ h}} \sbr{\ell((f'_{n} \circ h')(x), y) - \ell((f_{n} \circ h)(x), y) } > \frac{\epsilon}{2}$.
\end{definition}

\begin{definition}[Task-eluder dimension]
\label{def:task-eluder-dimension}
Given the setting of Definition~\ref{def:eps-independence}, for any representation $h \in \Hcal$, we denote by $\rho_h \rbr{\Fcal, \epsilon}$ the length of the longest sequence centered at $h$, $\cbr{(h,f_i)}_{i}$, such that each tuple is $\epsilon$-independent of its predecessors. \\

\vspace{-10pt}
\noindent Then, the $\epsilon$-task-eluder dimension of $(\Hcal, \Fcal)$ under probabilistic model $\PP$ and loss function $\ell$ is $\dim_{\PP, \ell}(\Hcal, \Fcal, \epsilon) := \sup_{h \in \Hcal} \rho_h(\Fcal, \epsilon)$. To avoid overloading the notation, we omit $\PP$ and $\ell$ and write $\dim(\Hcal, \Fcal, \epsilon)$ when they are clear from context.
\end{definition}

Intuitively, a task $(h, f_n)$ is independent of its predecessors if there is another representation $h'$  that is indistinguishable from $h$ on earlier tasks---because good prediction layers also exist under $h'$---yet $(h, f_n)$ provides new information that helps disambiguate between $h$ and $h'$, as there is no prediction layer for $h'$ that can match $(h, f_n)$. Here, ``equivalence'' between predictors is measured in terms of excess risk at scale $\Ocal(\epsilon)$ under the distributions induced by $\rbr{f_i \circ h}$'s.\footnote{In Section~\ref{sec:discussion}, we examine an alternative notion of $\epsilon$-independence (Definition~\ref{def:eps-independence}) and show, via a linear example, that it can fail to yield meaningful bounds.}

In other words, the task-eluder dimension quantifies how long the learner can continue picking up {\em at least some new information} from a possibly adversarial task sequence without pinning down a near-optimal representation. 
As a sanity check, we prove the following simple bound on the task eluder dimension when either $\Hcal$ or $\Fcal$ is finite, similar to the basic bounds of eluder dimension~\citep[][Proposition 1 therein]{osband2014model}: 

\begin{proposition}\label{prop:finite-eluder}
For any $\epsilon \ge 0$,
\[
\dim(\Hcal, \Fcal, \epsilon) \le 2\min \rbr{\abr{\Hcal}, \abr{\Fcal}}.
\] 
\end{proposition}

The proof of Proposition~\ref{prop:finite-eluder} is deferred to Appendix~\ref{app:finite-task-eluder}. 
In Section~\ref{sec:examples}, we provide examples of upper bounds on the task-eluder dimension for common classes.

\subsection{A closer look at Algorithm~\ref{alg:lifelong-nonparametric}}
Now that we are equipped with the notion of task-eluder dimension, we revisit Algorithm~\ref{alg:lifelong-nonparametric} to analyze its sample and space complexities.
For now, assume that the learner knows the task-eluder dimension $\dim(\Hcal, \Fcal, \epsilon) \le \Xi$ in advance and initializes $N = \Xi$ (line~\ref{line:initialize-N}); let us also disregard the steps in \textcolor{color3}{red} (lines~\ref{line:double-start} to \ref{line:double-end}).\footnote{See Appendix~\ref{app:pseudocode-known-task-eluder} for a clean version of the algorithm under knowledge of $\dim(\Hcal, \Fcal, \epsilon)$.} At the beginning of any task $t \ge 2$, suppose the counter $n = n_0$.

\begin{enumerate}[leftmargin=*]
    \item By choosing $m_N$ appropriately (as stated in Theorem~\ref{thm:main}) and leveraging data in the memory buffer for which the initial task and tasks where property test failed, 
    $\Mcal = \cbr{S_{t_1}, \ldots, S_{t_{n_0}}}$?, the learner ensures from multi-task ERM (line~\ref{line:multi-task-erm}) that the currently maintained $\hat{h}$ satisfies, with high probability,
    \[
    \sum_{i=1}^{n_0} \rbr{\Lcal_{\Pcal_{t_i}} (\check{f}^{(t_{n_0})}_i \circ \hat{h}) - \Lcal_{\Pcal_{t_i}} (f^*_{t_i} \circ h^*)} \le \epsilon. 
    \]

    \item Meanwhile, with $\tilde{m}$ chosen appropriately (as stated in Theorem~\ref{thm:main}), if the few-shot property test fails, then with high probability,
    \[
    \Lcal_{\Pcal_t}(f \circ \hat{h}) - \Lcal_{\Pcal_t}(f_t^* \circ h^*) > \frac{\epsilon}{2}, \quad \forall f \in \Fcal;
    \]
    in other words, $(h^*, f_t^*)$ is $\epsilon$-independent of $\cbr{(h^*, f^*_{t_i})}_{i=1}^{n_0}$.
\end{enumerate}
Since $\rho_{h^*}(\Fcal, \epsilon) \le \dim(\Hcal, \Fcal, \epsilon)$, the task-eluder dimension bounds the number of times this can happen. In other words, the task-eluder dimension characterizes the number of tasks for which the learner needs to acquire {\em additional} data beyond what is required for property testing, which then determines the sample and space complexities of the algorithm. 

\paragraph{Doubling trick.} In practice, the learner may not know $\dim(\Hcal, \Fcal, \epsilon)$ in advance. To address this, we utilize the doubling trick from the online learning literature \citep[e.g.,][]{shalev2012online}.
Specifically, we let $N$ be the running estimate of $\Xi$. We begin with a small estimate, $N = 1$ (line~\ref{line:initialize-N}). Each time it proves insufficient---i.e., for the current estimate $N$, property tests fail more than $N$ times---we double the estimate, clear the memory, and restart the process (lines~\ref{line:double-start} to~\ref{line:double-end}). This ensures that, without prior knowledge of $\dim(\Hcal, \Fcal, \epsilon)$, the number of times that the few-shot property test fails can still be bounded by $\Ocal \rbr{\dim(\Hcal, \Fcal, \epsilon)}$.

\section{Theoretical guarantees}
\label{sec:theoretical-guarantees}

In this section, we first provide additional background and introduce a few technical tools for our theoretical analysis (Section~\ref{sec:additional-background}). We then introduce two benchmark algorithms (Section~\ref{sec:benchmarks}) before presenting our main result (Section~\ref{sec:main-result}).

\subsection{Background: sample complexity of multi-task ERM} 
\label{sec:additional-background}

\paragraph{Capacities of $\Hcal$ and $\Fcal$ based on covering numbers.}
As is standard in statistical learning theory, sample complexity for generalization depends on the capacity of the learner's model class, often captured by notions such as VC dimension and Rademacher complexity \citep[e.g.,][]{shalev2014understanding}. Following \citep{baxter2000model}, we use {\em covering numbers} to characterize the capacities of $\Hcal$ and $\Fcal$ and later analyze sample complexity. We note that our findings in this work are not tied to this specific choice and can extend to other suitable complexity measures.

For any $f \in \Fcal$, let $f_\ell(z, y) := \ell(f(z), y)$ for any low-dimensional representation $z \in \Zcal$ and target $y \in \Ycal$, and let $\Fcal_\ell := \cbr{f_\ell: f \in \Fcal}$. For any distribution $Q$ on $\Zcal \times \Ycal$, let $d_Q(f_\ell, f'_\ell) := \int_{\Zcal \times \Ycal} \abr{f_\ell(z,y) - f'_\ell(z,y)} dQ(z,y)$ be the $L^1(Q)$ pseudo-metric on $\Fcal_\ell$. Then, for any $\epsilon_0 > 0$, we define the capacity of $\Fcal$ at scale $\epsilon_0$ to be
\begin{align*}
    \Ccal(\Fcal_\ell, \epsilon_0) := \sup_Q N(\Fcal_\ell, \epsilon_0, d_Q),
\end{align*}
where $N(\Fcal_\ell, \epsilon_0, d_Q)$ denotes the $\epsilon_0$-covering number of $\rbr{\Fcal_{\ell}, d_Q}$. For ease of notation, we often use $\Ccal(\Fcal, \epsilon_0)$ in place of $\Ccal(\Fcal_\ell, \epsilon_0)$.

For measure $P$ on $\Xcal \times \Ycal$, let $d_{P, \Fcal_\ell}(h,h') := \int_{\Xcal \times \Ycal} \sup_{f_\ell \in \Fcal_\ell} \abr{f_\ell(h(x), y) - f_\ell(h'(x), y)} dP(x,y)$ be a pseudo-metric on $\Hcal$. Then, for any $\epsilon_0 > 0$, we define the capacity of $\Hcal$ at scale $\epsilon_0$ to be
\begin{align*}
    \Ccal_{\Fcal_\ell}(\Hcal, \epsilon_0) := \sup_P N(\Hcal, \epsilon_0, d_{P, \Fcal_\ell}),
\end{align*}
where $N(\Hcal, \epsilon_0, d_{P, \Fcal_\ell})$ denotes the $\epsilon_0$-covering number of $\rbr{\Hcal, d_{P, \Fcal_\ell}}$. When the context is clear, we write
$\Ccal(\Hcal, \epsilon_0)$ for $\Ccal_{\Fcal_\ell}(\Hcal, \epsilon_0)$ to avoid clutter.

We note that $\log \Ccal(\Fcal, \epsilon_0)$ and $\log \Ccal(\Hcal, \epsilon_0)$ are often referred to as the metric entropy of $\Fcal$ and $\Hcal$, respectively \citep{haussler1992decision}. 
Section~\ref{sec:examples} provides examples of capacity bounds for $\Hcal$ and $\Fcal$.

\paragraph{Sample complexity of multi-task ERM.}
\citet{baxter2000model} established the following uniform bound for multi-task ERM using capacities based on covering numbers:
\begin{theorem}[\citealp{baxter2000model}, Theorem 4 and Theorem 6 thereof]
\label{thm:baxter-uniform-convergence}
Let $P_1, \ldots, P_n$ be the data-generating distributions of $n$ tasks. Let $\Hcal_0$ be a class of representations and $\Fcal_0$ a class of prediction layers. Suppose for each task $i$, an i.i.d.\ sample $S_i$ of size $m$ is drawn from $P_i^m$, where
\[
m \ge \Ocal \rbr{\frac{1}{n\epsilon_0^2} \rbr{\log \Ccal \big(\Hcal_0, \frac{\epsilon_0}{32}\big) + n \log \Ccal \big(\Fcal_0, \frac{\epsilon_0}{32} \big) + \log \frac{1}{\delta_0} }},
\]
then with probability at least $1 - \delta_0$, for any $(h, f_1, \ldots, f_n)$,
\[
\abr{\frac{1}{n} \sum_{i=1}^n \Lcal_{P_i}(f_i \circ h) - \frac{1}{n} \sum_{i=1}^n \widehat{\Lcal}_{S_i}(f_i \circ h)} \le \epsilon_0.
\]
\end{theorem}
\begin{corollary}
\label{cor:baxter-multitask-sample}
It follows immediately that if the number of samples per task exceeds

\[
\Ocal \rbr{\frac{1}{n\epsilon_0^2} \rbr{\log \Ccal \big(\Hcal_0, \frac{\epsilon_0}{64}\big) + n \log \Ccal \big(\Fcal_0, \frac{\epsilon_0}{64} \big) + \log \frac{1}{\delta_0} }},
\]
then with probability at least $1 - \delta_0$,
\[
\frac{1}{n} \sum_{i=1}^n \Lcal_{P_i}(\hat{f}_i \circ \hat{h}) \le \min_{h, f_1, \ldots, f_n} \frac{1}{n} \sum_{i=1}^n \Lcal_{P_i}(f_i \circ h) + \epsilon_0,
\]
where $(\hat{h}, \hat{f}_1, \ldots, \hat{f}_n)$ is the solution to multi-task ERM (Eq.~\eqref{eqn:multi-task-erm}).
\end{corollary}
This result highlights the benefit of MTL. With only one task, the learner bears the full burden of learning both the representation and the prediction layer. Since we are mostly interested in the regime where $\log \Ccal \big(\Hcal_0, \frac{\epsilon_0}{64}\big) \gg \log \Ccal \big(\Fcal_0, \frac{\epsilon_0}{64}\big)$, learning the representation individually for each task can be costly. With MTL, while the learner still has to learn $f_1, \ldots, f_n$, the cost of learning the shared representation is amortized over the tasks.

\subsection{Warm-up: two benchmarks} 
\label{sec:benchmarks}
To better highlight the performance of our algorithm, we introduce two benchmark methods for the lifelong representation learning problem.
The first is a naive {\em baseline} algorithm in which the learner simply ignores any shared structure and solves each task independently using single-task ERM, which we refer to as \indep. 
Applying Corollary~\ref{cor:baxter-multitask-sample} (with $n = 1$) and the union bound, we obtain the following sample complexity bound of \indep,
\begin{align} \label{eqn:baseline-guarantee}
\tilde{\Ocal} \rbr{\frac{T}{\epsilon^2} \rbr{ \log \Ccal \big(\Hcal, \frac{\epsilon}{64} \big) + \log \Ccal \big(\Fcal, \frac{\epsilon}{64} \big)}}. 
\end{align}
Alternatively, had the learner known $h^*$ beforehand, it suffices to only learn the prediction layer for each task by solving single-task ERM with a singleton representation class $\cbr{h^*}$ (using a much smaller sample). We call this {\em skyline} algorithm \oracle. Formally, it follows from Corollary~\ref{cor:baxter-multitask-sample} and the union bound that \oracle has a much lower sample complexity of
\begin{align}
    \label{eqn:oracle-guarantee}
    \tilde{\Ocal} \rbr{\frac{T}{\epsilon^2} \log \Ccal \big(\Fcal, \frac{\epsilon}{64} \big)}.
\end{align}

\subsection{Main result}
\label{sec:main-result}
We now present our main theorem. In the interest of space, its proof is deferred to Appendix~\ref{app:proofs-main}.
\begin{theorem}
\label{thm:main}
Let $\Xi = \dim(\Hcal, \Fcal, \epsilon) < \infty$. 
Suppose $\Hcal$ and $\Fcal$ have finite capacities; that is, $\Ccal \rbr{\Hcal, \frac{\epsilon}{256\Xi}} < \infty$ and $\Ccal \rbr{\Fcal, \frac{\epsilon}{256\Xi}} < \infty$. 
In Algorithm~\ref{alg:lifelong-nonparametric}, set
\[
m_N = \tilde{\Ocal} \rbr{\frac{N}{\epsilon^2} \sbr{\log \Ccal \rbr{\Hcal, \frac{\epsilon}{64N}} + N \log \Ccal\rbr{\Fcal, \frac{\epsilon}{64N}}}}
\]
for each $N$, and set
\[
\tilde{m} = \tilde{\Ocal} \rbr{\frac{1}{\epsilon^2} \log \Ccal\rbr{\Fcal, \frac{\epsilon}{128}}},
\]
where $\tilde{\Ocal}$ hides logarithmic factors in $T$ and $\frac{1}{\delta}$.
Then, with probability at least $1-\delta$,
\begin{itemize}[leftmargin=*,itemsep=0pt,topsep=2pt]
    \item[-] For every task, Algorithm~\ref{alg:lifelong-nonparametric} outputs a predictor with excess risk at most $\epsilon$;
    \item[-] Algorithm~\ref{alg:lifelong-nonparametric} performs multi-task ERM at most $\Ocal \rbr{\Xi}$ times;
    \item[-] The sample complexity is bounded by
        \begin{align} \label{eqn:alg-sample-complexity}
         \tilde{\Ocal} \bigg( \underbrace{\frac{T}{\epsilon^2} \log \Ccal\rbr{\Fcal, \frac{\epsilon}{128}}}_{\text{cost of few-shot tests}} +  \underbrace{\frac{\Xi^2}{\epsilon^2} \sbr{\log \Ccal\rbr{\Hcal, \frac{\epsilon}{128\Xi}} + \Xi \log \Ccal\rbr{\Fcal, \frac{\epsilon}{128\Xi}}}}_{\text{overhead of lifelong representation learning}}\bigg).
        \end{align}
        In addition, the size of the memory buffer it requires is at most 
        \[
        \tilde{\Ocal} \rbr{\frac{\Xi^2}{\epsilon^2} \sbr{\log \Ccal\rbr{\Hcal, \frac{\epsilon}{128\Xi}} + \Xi \log \Ccal\rbr{\Fcal, \frac{\epsilon}{128\Xi}}}}.
        \]
\end{itemize}
\end{theorem}
See Theorem~\ref{thm:restate-main} in Appendix~\ref{app:proofs-main} for a restatement of Theorem~\ref{thm:main} with all constants specified.
To interpret the sample complexity bound in Eq.~\eqref{eqn:alg-sample-complexity}, observe that the first term accounts for the sample complexity from property testing. It is necessary even with a priori knowledge of $h^*$ albeit a constant factor in the scale of the capacity (cf. the guarantee of \oracle in Eq.~\eqref{eqn:oracle-guarantee}). The second term reflects the overhead of learning the unknown representation. It is governed by the capacities of $\Hcal$ and $\Fcal$, as well as the task-eluder dimension $\Xi = \dim(\Hcal, \Fcal, \epsilon)$. Since $\dim(\Hcal, \Fcal, \epsilon)$ measures the complexity of the model class under $\PP$ and $\ell$, it does not grow with $T$. This highlights the benefit of representation transfer when compared to the baseline in Eq.~\eqref{eqn:baseline-guarantee}.

In particular, as $T$ approaches infinity in a truly {\em lifelong} setting, the cost of representation learning becomes negligible, and the sample complexity nearly matches that of \oracle. Similarly, the space complexity (size of the memory buffer) also remains bounded as $T$~grows.

\section{Examples}
\label{sec:examples}

We now provide two concrete examples of how our results can be specialized in regression and classification tasks. 
Let $\Xcal = \cbr{x \in \RR^d: \nbr{x} \le 1}$. We focus on a class of low-dimensional linear representations, $\Hcal = \cbr{x \mapsto B^\top x: B \in \RR^{d \times k}, B^\top B = I_k}$, where $k \ll d$. This class has been widely studied to demonstrate the benefit of representation transfer \citep[e.g.,][]{balcan2015efficient,hu2021near, tripuraneni2021provable,du2021fewshot}.
In the following, we sometimes abuse notation and identify a function (either representation or prediction layer) with its parameter. 
Due to space constraints, our proofs are deferred to Appendix~\ref{app:proofs-examples}.

\paragraph{Linear regression with noise.} 
Let $\Fcal^{\mathrm{lin}} := \cbr{z \mapsto w^\top z: w \in \RR^k, \nbr{w} \le \frac{1}{2}}$ be a class of linear functions, and $\Ycal = [-1,1]$. Consider the following probabilistic model $\PP$: $P_X$ over $\Xcal$ satisfies $I \precsim \EE_{x \sim P_X} \sbr{xx^\top} \precsim I$. For any $h$ and $f$, given an input $x$, $y = (f \circ h)(x) + \eta$, where $\eta$ is independently drawn from a shared noise distribution with support $[-\frac12,\frac12]$, mean zero, and variance $\kappa$. The $T$ tasks are well-specified with ground truth representation $B^*$ and prediction layers $w_1^*, \ldots, w_T^*$. Let $\ell(y', y) := \frac{1}{4} (y' - y)^2$ which has range $[0,1]$.

\begin{proposition}
\label{prop:example-linear-regression}
Let $\epsilon \in (0,1)$. We have
\[
\log \Ccal(\Fcal^{\mathrm{lin}}_{\ell}, \epsilon) \le \Ocal \Big(k \log\frac{1}{\epsilon}\Big), \ \log \Ccal_{\Fcal^{\mathrm{lin}}_{\ell}}(\Hcal, \epsilon) \le \Ocal \Big(dk \log \frac{1}{\epsilon} \Big), \
\dim_{\PP, \ell}(\Hcal, \Fcal^{\mathrm{lin}}, \epsilon) \le \Ocal \big(k \log \frac{1}{\epsilon}\big).
\]
\end{proposition}

\paragraph{Classification with logistic regression.}
Let $\Ycal = \cbr{0, 1}$. Denote by $\sigma(v) = \frac{1}{1 + e^{-v}}$ the logistic sigmoid function and let $\Fcal_{\mathrm{log}} := \cbr{z \mapsto \sigma(w^\top z): w \in \RR^k, \nbr{w} \le \frac 14}$. Data-generating distributions in $\PP$ are defined as follows: $P_X$ satisfies $I \precsim \EE[xx^T] \precsim I$, and for each $f$ and $h$, $P(y = 1\ |\ x; f \circ h) = \sigma((f \circ h)(x))$. The $T$ tasks are well specified by $B^*$ and $w_1^*, \ldots, w_T^*$. Let $\ell(y',y) = -y \log y' - (1 - y) \log ( 1- y')$.
\begin{proposition}
\label{prop:example-logistic}
Let $\epsilon \in (0,1)$. We have
\[
\log \Ccal(\Fcal^{\mathrm{log}}_\ell, \epsilon) \le \Ocal \Big(k \log\frac{1}{\epsilon}\Big), \ \log \Ccal_{\Fcal^{\mathrm{log}}_\ell}(\Hcal, \epsilon) \le \Ocal \Big(dk \log \frac{1}{\epsilon} \Big), \
\dim_{\PP, \ell}(\Hcal, \Fcal^{\mathrm{log}}, \epsilon) \le \Ocal \big(k \log \frac{1}{\epsilon}\big).
\]
\end{proposition}

In these examples, it follows straightforwardly from Theorem~\ref{thm:main} that the sample complexity of Algorithm~\ref{alg:lifelong-nonparametric} is $\tilde{\Ocal} \rbr{\rbr{kT + dk^3}/{\epsilon^2}}$. When $T$ is large, this bound is dominated by $\tilde{\Ocal} \rbr{kT/\epsilon^2}$ which only depends on $k \ll d$. This highlights the benefit of learning and leveraging the shared structure. 
We note that our goal here is to showcase the applicability of Theorem~\ref{thm:main}, rather than deriving the sharpest bound compared to what could be achieved with specialized techniques. In Appendix~\ref{app:proofs-examples}, we also discuss how our analysis can apply to classification with random classification noise and the $0$-$1$ loss, via a similar argument to Theorem~\ref{thm:main} using multi-task ERM guarantees based on VC~dimension.

\section{Related work}
\label{sec:related-work}

Multi-task learning \citep{caruana1997multitask} has been extensively studied in the literature; see \citep{zhang2021survey} for a survey. Lifelong learning can be traced back to \citep{thrun1995lifelong}. Since then, it has seen applications in a wide range of domains, such as robotics \citep{lowrey2018plan}, computer vision \citep{rebuffi2017icarl}, and natural language processing \citep{de2019episodic}. See \citep{sodhani2022introduction} for an introduction to various approaches to lifelong learning. Techniques from MTL and LTL have also been extended to lifelong learning \citep{finn2019online}. Much of the literature focuses on mitigating catastrophic forgetting \citep{mccloskey1989catastrophic}, whereas our main objective is to provide a theoretical study on the benefit of representation transfer.

In particular, our work builds upon a line of research that studies sample complexity guarantees of MTL and LTL where tasks share a common representation. \citet{baxter2000model} study a general framework for inductive bias learning and establish guarantees based on covering numbers of the hypothesis space family. \citet{maurer2016benefit} use Gaussian complexities with a chain rule \citep{maurer2016chain} to derive data-dependent bounds. \citet{tripuraneni2020theory} introduce a notion of task diversity which characterizes when transfer learning of representations can be achieved in LTL. \citet{xu2021representation} study the setting where source and target tasks may use different classes of prediction layers. \citet{watkins2023optimistic} establish optimistic rates that adapt to the difficulty of a target task. These papers all consider the composite model $f \circ h$ that we study in this work.

MTL and LTL with shared linear representations have been studied \citep{tripuraneni2021provable,du2021fewshot,chen2022active,aliakbarpour2024metalearning}. \citet{pentina2015multi} study MTL and lifelong learning of kernels. Sequential and parallel representation transfer have also been explored for linear bandits \citep{hu2021near,yang2021impact,qin2022non,duong2024beyond}.

Our work is most directly related to \citep{balcan2015efficient} and \citep{cao2022provable}. Both papers focus on lifelong learning for binary classification in the noiseless, realizable setting with linear prediction layers. \citet{balcan2015efficient} study low-dimensional linear representations, whereas \citet{cao2022provable} also consider one-hidden-layer neural networks. We note that the guarantee for nonlinear representations in \citep{cao2022provable} relies on an assumption (Assumption 1 therein) that for two maps $u$ and $v$,
$d(u, v) \lesssim \Pr_{P_X} \rbr{\sign(u(x)) \neq \sign(v(x))} \lesssim d(u,v)$,
where $d(\cdot, \cdot)$ denotes the angle. While this is true for linear maps when $P_X$ is isotropic, log-concave, the first inequality may not hold in general for one-hidden-layer neural
networks under the same assumption. Algorithmically, both papers dynamically expand their representation with \citep{cao2022provable} additionally performing refinement. In contrast, we use multi-task ERM to update our representation when it becomes~insufficient.

\citet{li2022provable} propose an architecture-based algorithm for continual representation learning. They provide sample complexity bounds under an assumption of {\em sequential task diversity}: earlier tasks are diverse enough to ensure small representational mismatch for new tasks. Similar assumptions have been made in \citep{tripuraneni2020theory} for LTL and \citep{qin2022non} for sequential transfer in linear bandits. In contrast, we consider an online setting where the task sequence may be chosen by an adversary, and the learner needs to carefully manage when to request more data.

\citet{alquier2017regret} study lifelong representation learning in an \textit{online-within-online} setting: tasks arrive sequentially, and within each task, data points are also revealed sequentially with the learner predicting each instance. They also consider a \textit{batch-within-online} setting, where data from each task is made available all at once. \citet{alquier2017regret} study and establish compound regret bounds of the learner. In contrast, we allow the learner to request samples over multiple rounds, and require that, with high probability, it outputs for every task a predictor with excess risk at most $\epsilon$.

\section{Empirical validation}
\label{sec:empirical-validation}

In this section, we empirically validate our theoretical results using synthetic and semi-synthetic data. In particular, we focus on the following questions:
\begin{enumerate}[leftmargin=*,itemsep=2pt,topsep=4pt]
    \item Is the number of updates to the internal representation maintained by Algorithm~\ref{alg:lifelong-nonparametric} indeed bounded by the task-eluder dimension, $\mathcal{\dim(\mathcal{H}, \mathcal{\Fcal}, \epsilon)}$?

    \item In practice, does Algorithm~\ref{alg:lifelong-nonparametric} require only a limited number of multi-task ERM calls (representation updates) for more expressive function classes beyond the examples given in Section~\ref{sec:examples}?
\end{enumerate}
To address these questions, we performed experiments in three settings. Implementation details are provided in Appendix~\ref{app:implementation-details}.

\subsection{Synthetic linear and semi-synthetic MNIST experiments}

\paragraph{Synthetic logistic regression tasks.}  We first consider a binary logistic regression setting with synthetic data. 
Let $\theta_1^*, \ldots, \theta_T^* \in \mathbb{R}^d$ denote the parameters associated with $T$ tasks, and let $\beta > 0$ be a parameter that governs the noise level of the tasks. There exists a $k$-dimensional shared representation, given by a semi-orthogonal matrix $B^* \in \mathbb{R}^{d \times k}$, such that for each $t \in T$, $\theta^*_t = B^* w_t^*$ for some $w^*_t \in \mathbb{R}^k$, where $\| w^*_t \| = \beta$. For each task $t$, the covariates $x \in \mathbb{R}^d$ are drawn from $\mathcal{N}(0, I)$, and the labels are generated such that $\Pr(y = 1 | x) = \sigma( x^\top \theta^*_t )$, where $\sigma$ denotes the logistic sigmoid function. We set $\ell$ to be the binary cross-entropy loss.

We set the parameters to $d = 10$, $T = 50$, and target excess risk $\epsilon = 0.05$. We varied $k \in \cbr{3,5,8}$ and $\beta \in \cbr{1,4,8}$, and for each configuration we ran $10$ trials, with $B^*$ and $w^*_t$'s generated randomly. We implemented Algorithm~\ref{alg:lifelong-nonparametric} with $N$ initialized as $k \log \frac{1}{\epsilon}$, based on Proposition~\ref{prop:example-logistic}. We set 
\[
m_N = \frac{1}{\epsilon^2} (dk + kN) \log \frac{1}{\epsilon}, \qquad \text{and} \qquad \tilde{m} = \frac{k}{\epsilon^2} \log \frac{1}{\epsilon}.
\]
Note that the choice of $m_N$ is smaller than what Theorem~\ref{thm:main} requires, $\tilde{O}\rbr{\frac{k}{\epsilon^2} (dk + kN)}$, which highlights the practical efficiency of our algorithm. 

To certify correctness, we evaluated each predictor output on a held-out dataset of size $\frac{32}{\epsilon^2}$ to verify that its excess risk is at most $\epsilon$. Note that this step is used only for evaluation and is not part of the algorithm.

\begin{figure}[t]
  \centering
  \setlength{\tabcolsep}{2pt}
  \begin{tabular}{ccc}

    \begin{subfigure}[t]{0.33\linewidth}
      \centering
      \includegraphics[width=\linewidth]{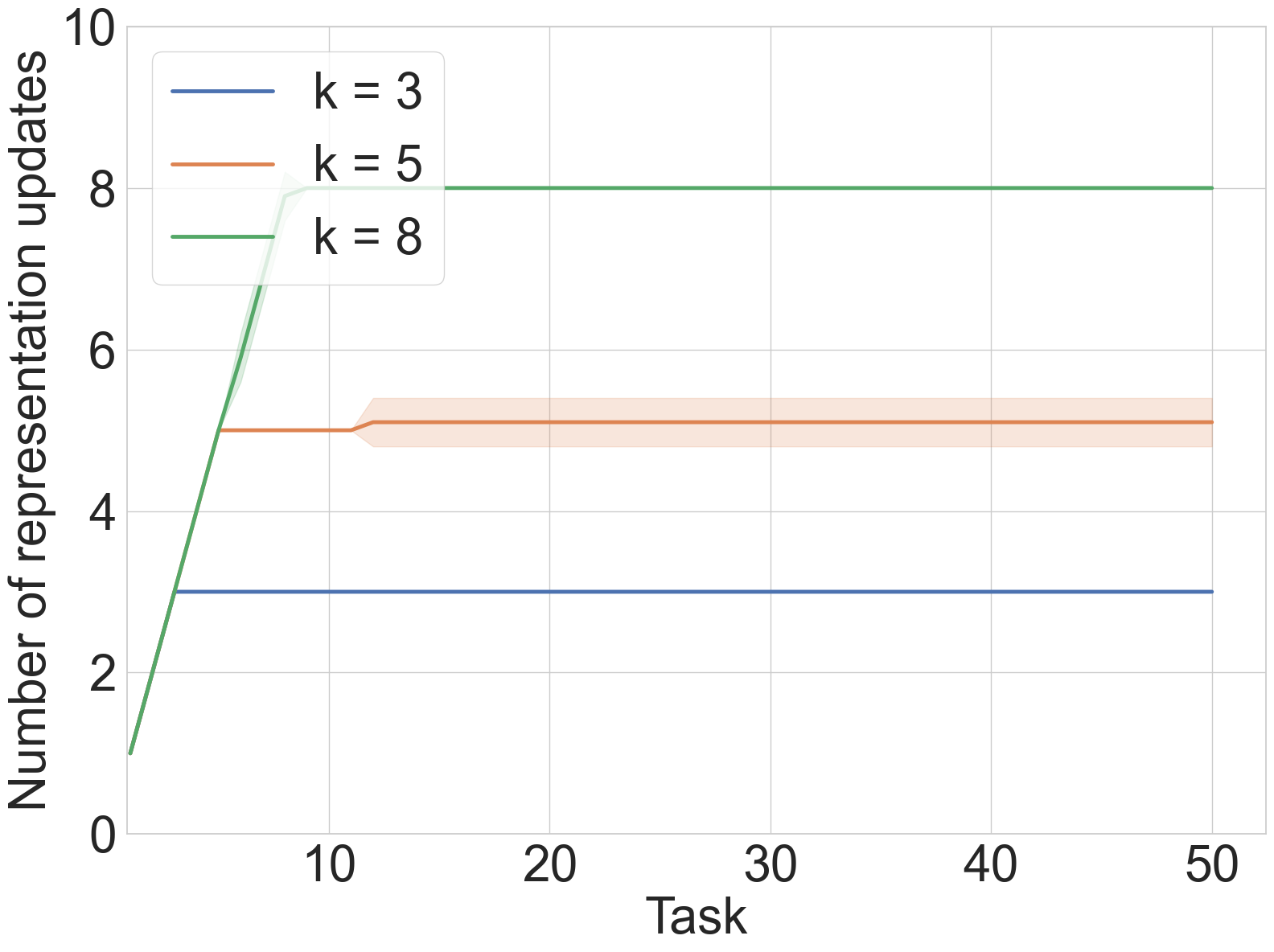}\\[3pt]
      \includegraphics[width=\linewidth]{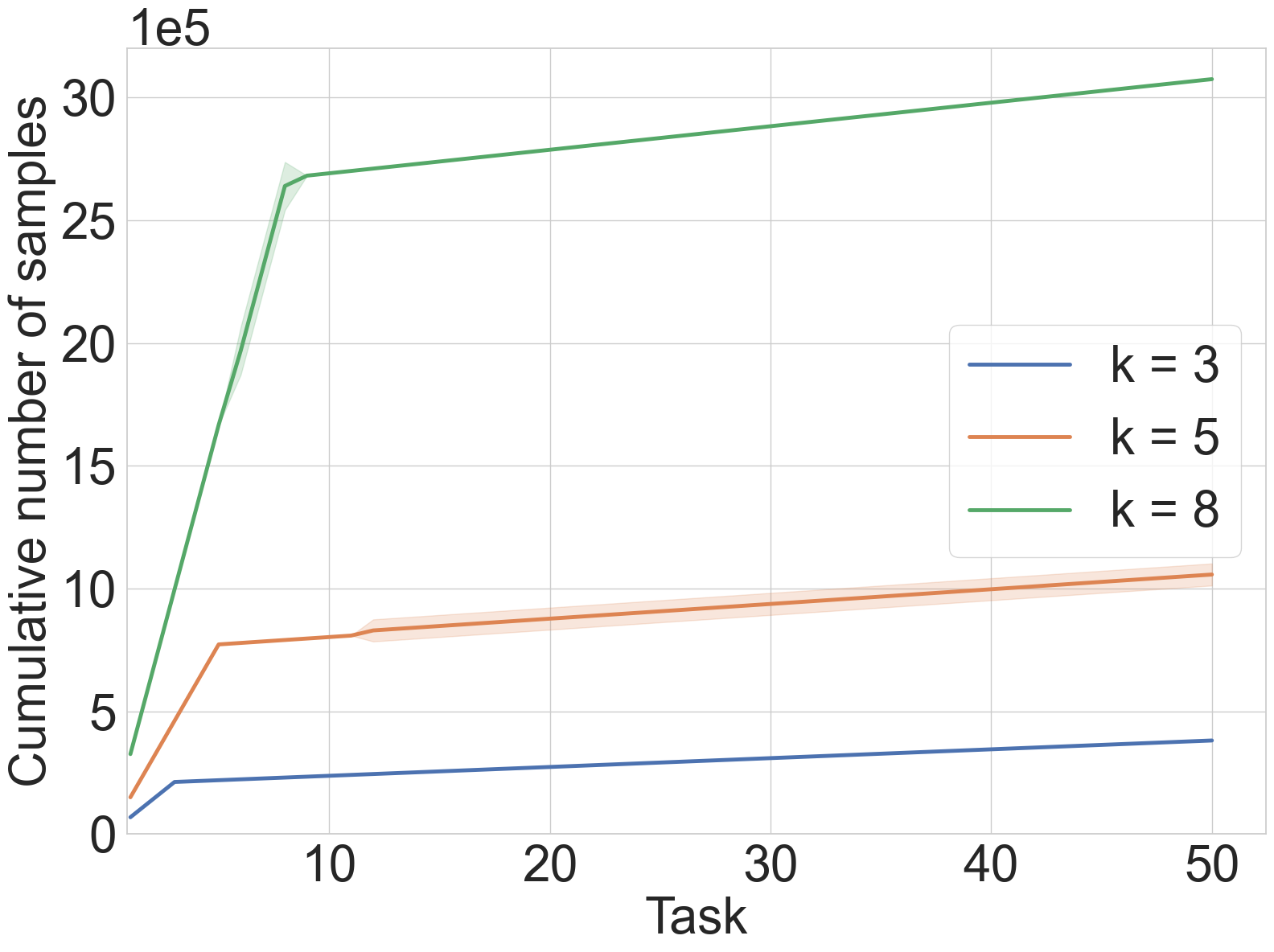}
      \caption{Low noise ($\beta = 8$)}
    \end{subfigure}
    
    \begin{subfigure}[t]{0.33\linewidth}
      \centering
      \includegraphics[width=\linewidth]{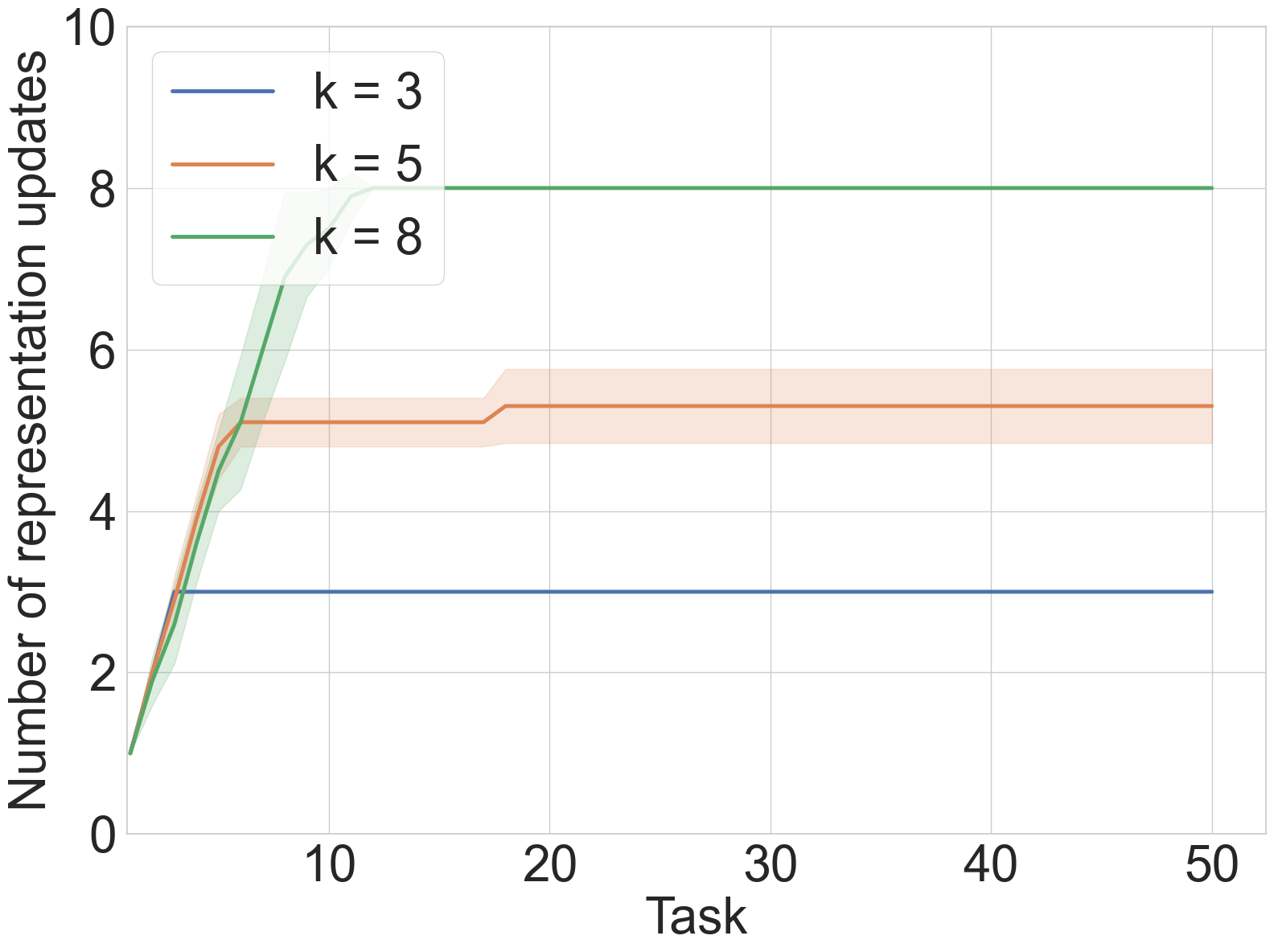}\\[3pt]
      \includegraphics[width=\linewidth]{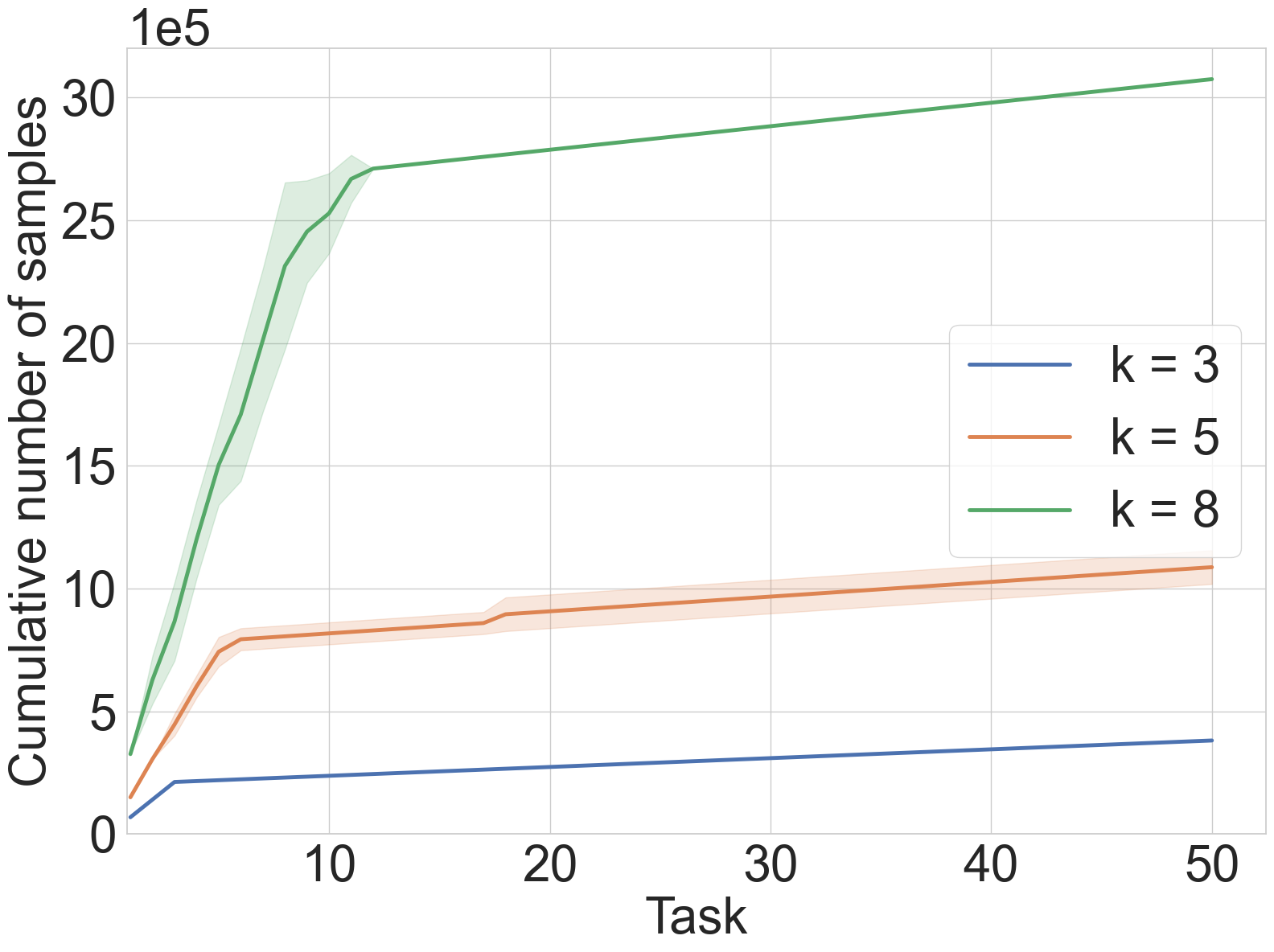}
      \caption{Medium noise ($\beta = 4$)}
    \end{subfigure}
    
    \begin{subfigure}[t]{0.33\linewidth}
      \centering
      \includegraphics[width=\linewidth]{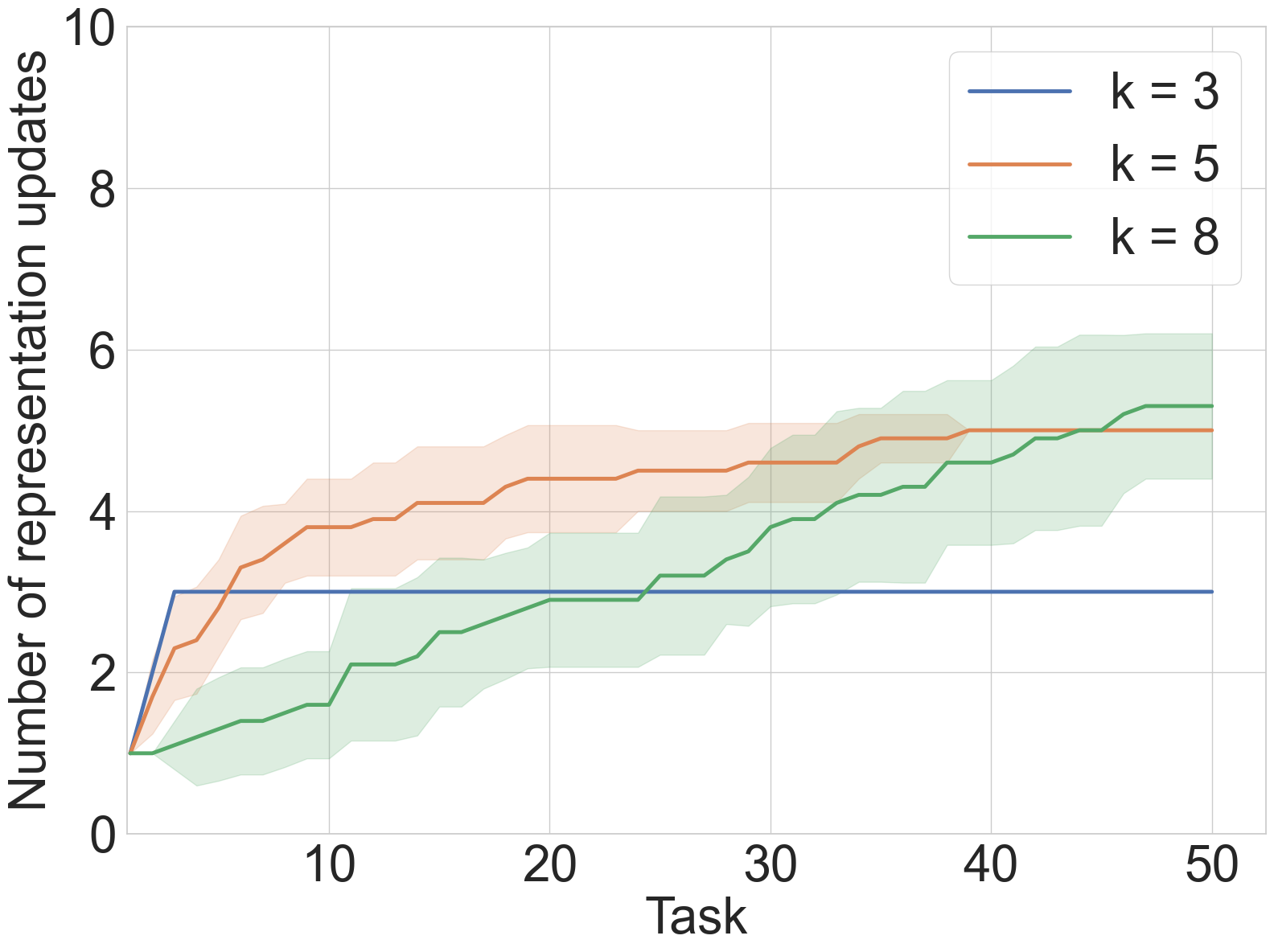}\\[3pt]
      \includegraphics[width=\linewidth]{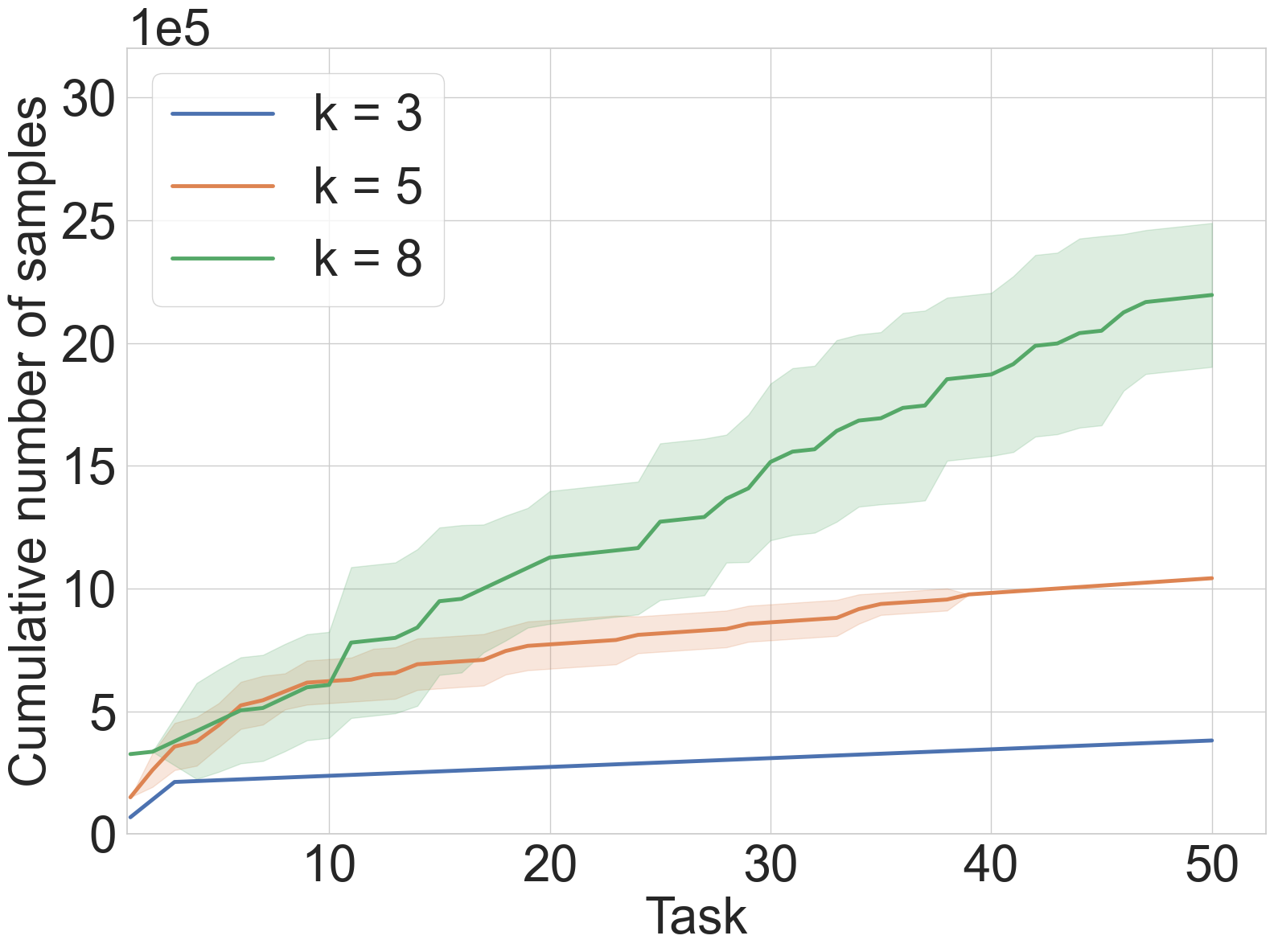}
      \caption{High noise ($\beta = 1$)}
    \end{subfigure}

  \end{tabular}
  \vspace{4pt}
  \caption{
    Results under different noise levels. For each noise level, the top plot shows the average cumulative number of samples used over $50$ tasks for each value of $k$, and the bottom plot shows how the cumulative number of representation updates evolve over the tasks. Shaded regions denote one standard deviation.
  }
  \label{fig:linear-exp}
\end{figure}

\begin{table}
\centering
\caption{Total number of representation updates performed by Algorithm~\ref{alg:lifelong-nonparametric} at each noise level.}
\label{tab:linear-exp-results}
\begin{tabular}{@{}lccc@{}} %
\specialrule{1.2pt}{0pt}{2pt}
 & Low noise ($\beta = 8$) & Medium noise ($\beta = 4$) & High noise ($\beta = 1$) \\
\midrule
$k=3$                       & $3 \pm 0$     & $3 \pm 0$     & $3 \pm 0$ \\
$k=5$                       & $5.1 \pm 0.3$ & $5.3 \pm 0.46$ & $5 \pm 0$ \\
$k=8$                       & $8.0 \pm 0$   & $8.0 \pm 0$   & $5.3 \pm 0.9$ \\
\specialrule{1.2pt}{2pt}{0pt}
\end{tabular}
\end{table}

Table~\ref{tab:linear-exp-results} reports the total number of representation updates performed by our algorithm (including task $1$), and the top row of Figure~\ref{fig:linear-exp} shows how the number of updates increases over tasks for each configuration. Across all regimes, the number of updates closely tracked the dimension of the representation $k$. This is consistent with Theorem~\ref{thm:main}. 
An interesting observation occurs in the high-noise regime, where the Bayes optimal risk is larger and so the absolute target risk is higher. The learner appeared to have not needed to fully uncover the underlying representation within $T = 50$ tasks when $k = 8$, as is seen in Figure~\ref{fig:linear-exp} which shows the cumulative number of samples and number of representation updates over the tasks for each configuration. 

\paragraph{Binary classification with MNIST digits.}
We also evaluated our algorithm in a semi-synthetic setting based on the MNIST handwritten digit database \citep{deng2012mnist}. Each task is defined as a binary classification problem, formed by randomly selecting one digit as the positive class and five other digits as the negative class. 
We flatten each $28 \times 28$ digit image so that $d = 784$, and we consider shared representations given by one-hidden layer neural networks,
\[
h(x) = \mathrm{ReLU}(B^\top x), \qquad B \in \RR^{d \times k}, \qquad k = 128.
\]
For each task $t$, prediction is made using a linear prediction layer $w_t \in \RR^k$ as $\mathds{1}\cbr{\inner{w_t}{h(x)} \ge 0}$.

We set $T= 50$ tasks and evaluated the performance of the learner under the $0$-$1$ loss with a target absolute error of $\epsilon = 0.05$. The learner uses only the MNIST training set for multi-task ERM and few-shot property tests\footnote{During training, the binary cross-entropy loss is used as a surrogate loss function.}. For few-shot property tests, $800$ samples are drawn from each of the positive and negative classes. If the empirical risk of the learned predictor is below $\frac{2}{3}\epsilon$, the test succeeds and the learner moves on; otherwise, the representation is updated via multi-task ERM. For each representation update (including the first task), the learner draws $m$ samples from the positive class and $m$ samples from the negative class, where $m$ is the number of available training images for the positive class (around $6000$), and applies multi-task ERM using these samples together with those stored in the memory.

Across $10$ trials, the learner updated its representation for an average of $4.9$ times, with a standard deviation of $0.54$. Figure~\ref{fig:mnist-num-updates} illustrates the average cumulative number of updates over $50$ tasks. 
Observe that the number of updates grew sublinearly with the number of tasks. This suggests that the learner quickly learned a representation that enabled accurate linear predictors to be learned from only a small number of examples.
To verify the correctness of our algorithm, we evaluated the predictors produced by the algorithm on data from the MNIST test set\footnote{This evaluation step is not part of the algorithm.
}. Figure~\ref{fig:mnist-errors} shows that the predictors consistently had $0$-$1$ risks below $\epsilon = 0.05$.

\begin{figure}[t]
  \centering
  \begin{subfigure}[t]{0.42\linewidth}
    \centering
    \includegraphics[width=\linewidth]{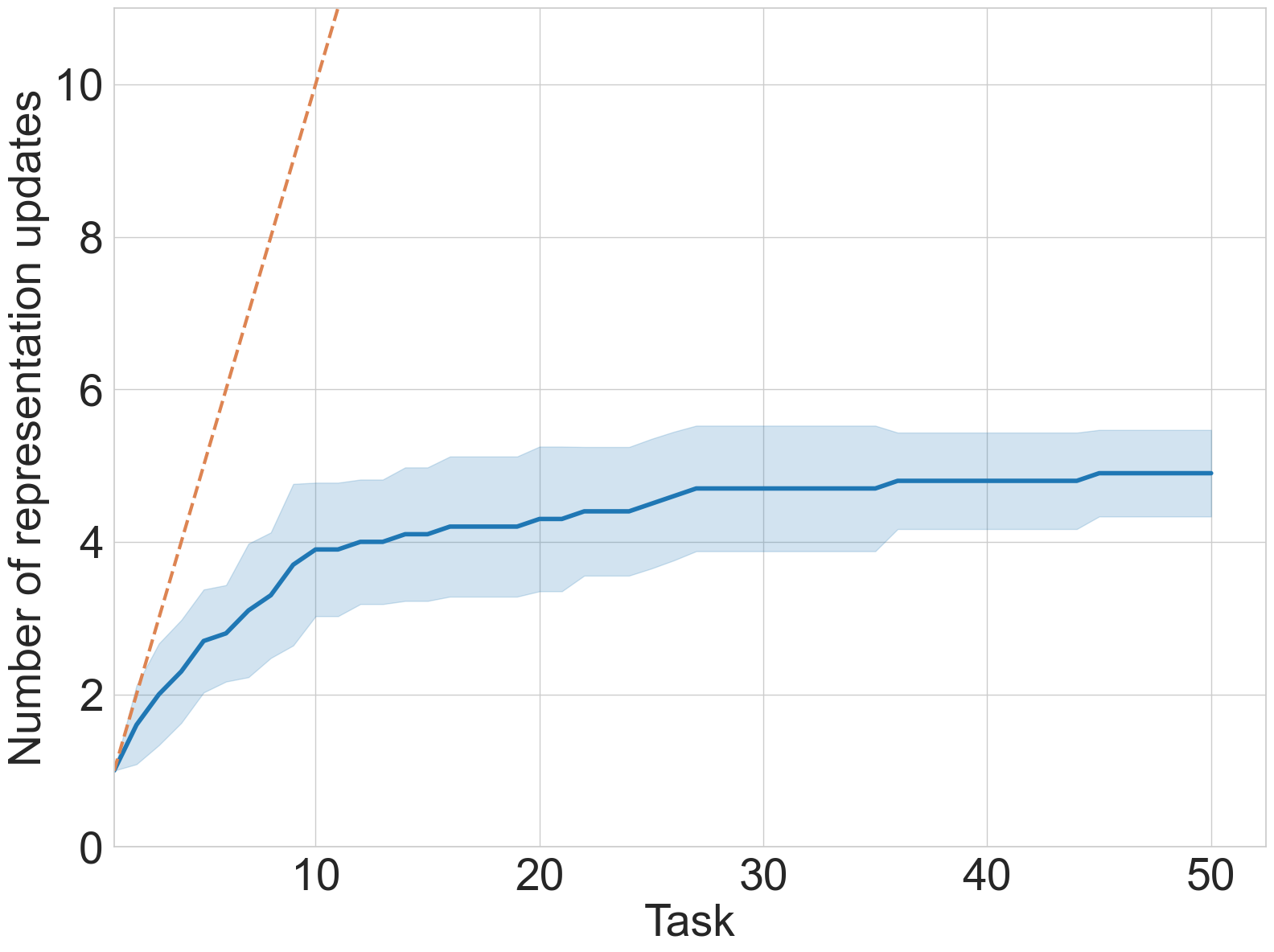}
    \caption{}
    \label{fig:mnist-num-updates}
  \end{subfigure}
  \hspace{40pt}
  \begin{subfigure}[t]{0.42\linewidth}
    \centering
    \includegraphics[width=\linewidth]{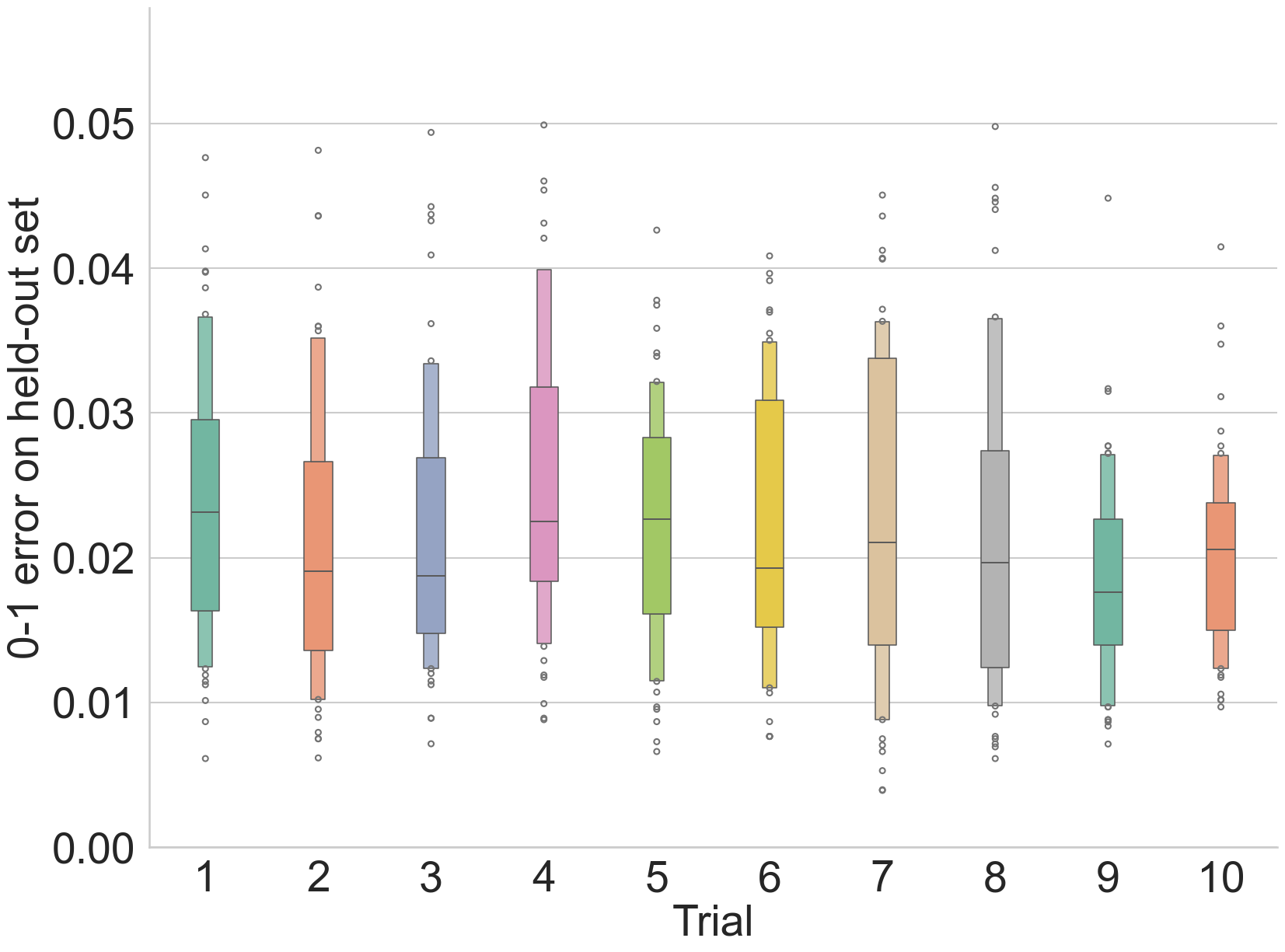}
    \caption{}
    \label{fig:mnist-errors}
  \end{subfigure}
  \caption{Performance on semi-synthetic experiments with MNIST digits. (a) The solid curve shows, on average, how the number of representation updates increases over $50$ binary digit classification tasks, with the shaded area showing one standard deviation. The dashed line represents linear growth. 
  (b) Each box plot shows the distribution of $0$-$1$ errors of the $50$ produced predictors when evaluated on held-out data from the MNIST test set in one of the $10$ independent trials.}
  \label{fig:mnist-exp}
\end{figure}

\begin{figure}[t]
  \centering
  \begin{subfigure}[t]{0.42\linewidth}
    \centering
    \includegraphics[width=\linewidth]{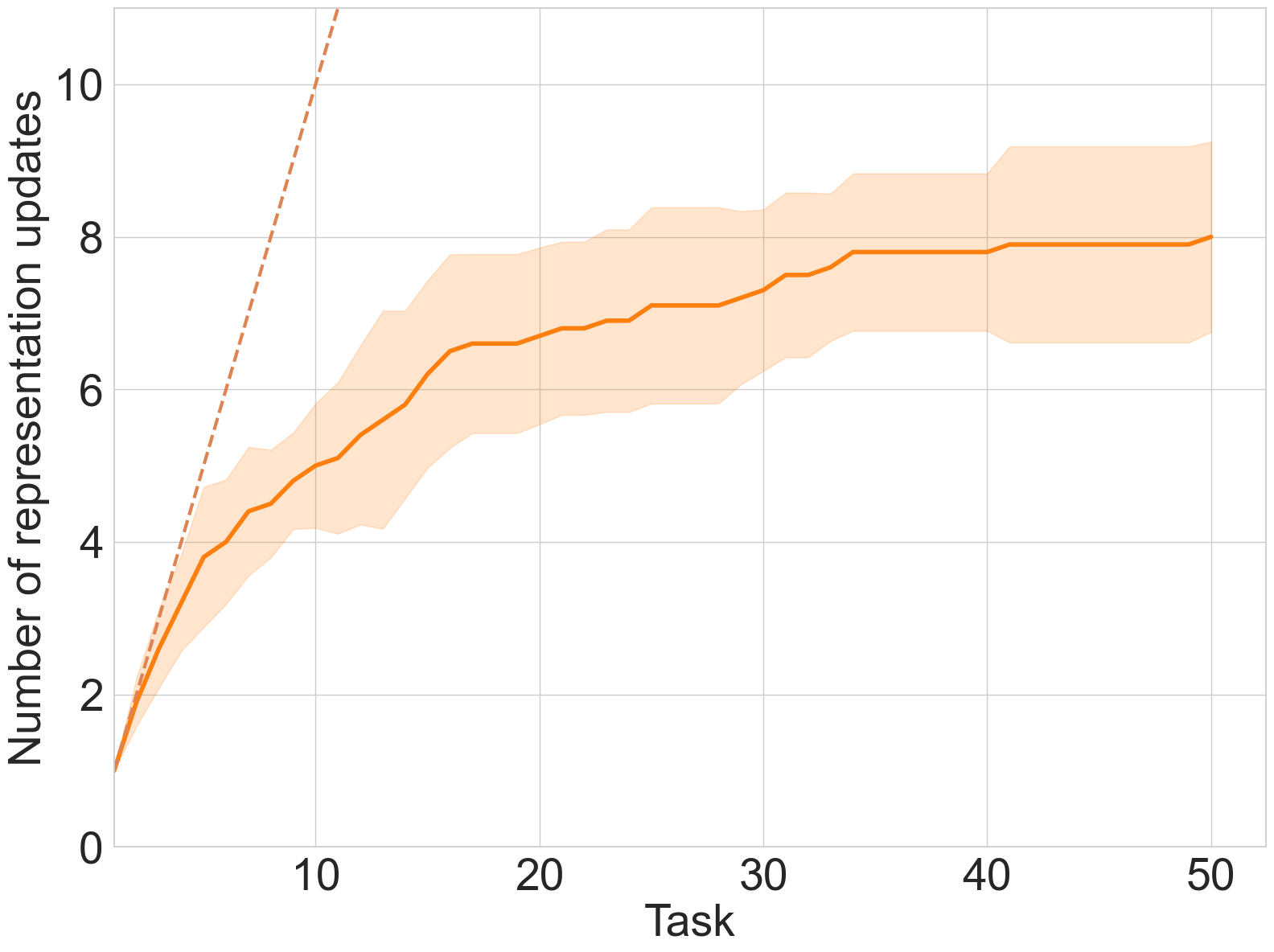}
    \caption{}
    \label{fig:cifar10-num-updates}
  \end{subfigure}
  \hspace{40pt}
  \begin{subfigure}[t]{0.42\linewidth}
    \centering
    \includegraphics[width=\linewidth]{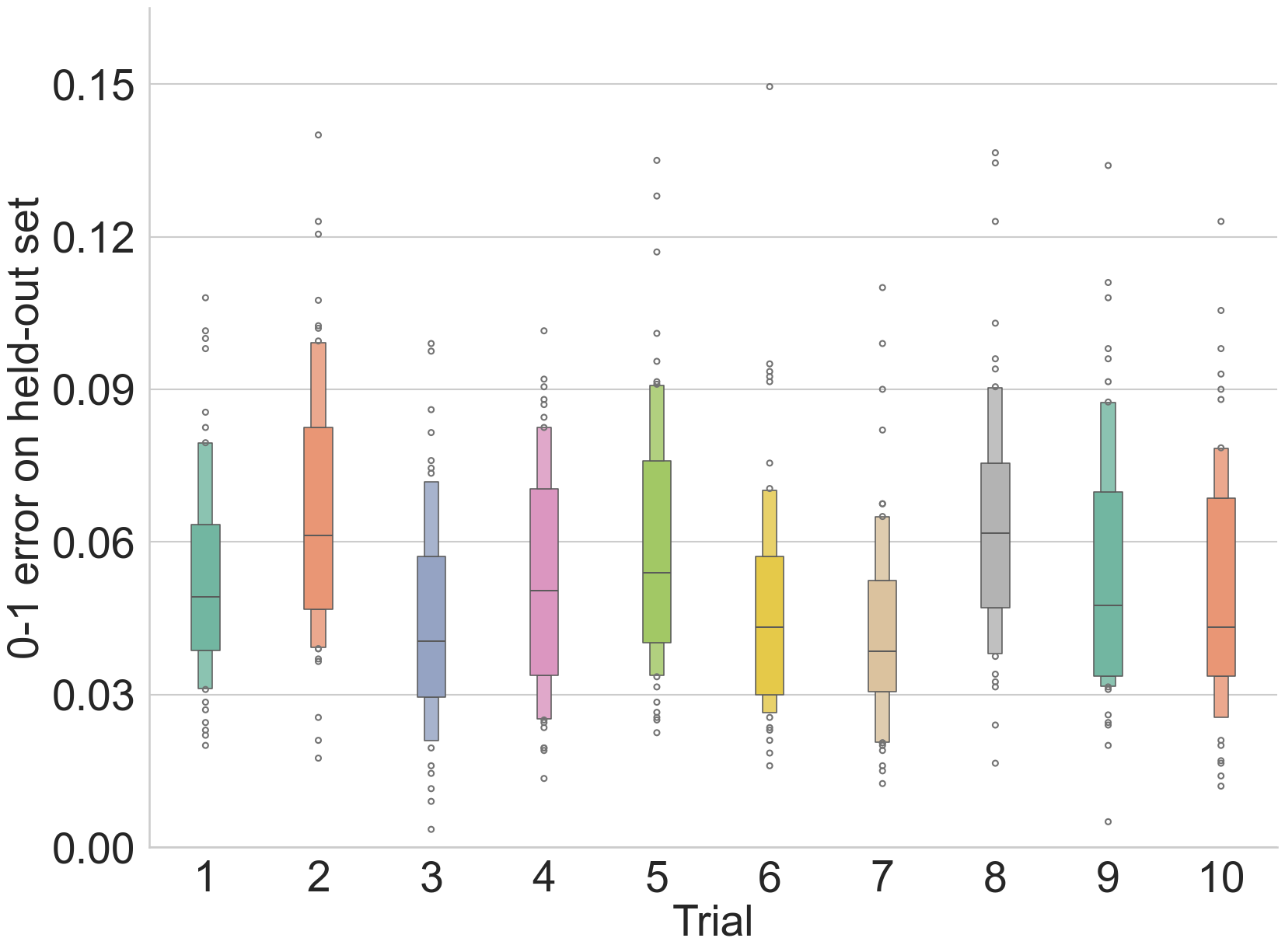}
    \caption{}
    \label{fig:cifar10-errors}
  \end{subfigure}
  \caption{Performance on semi-synthetic experiments with CIFAR-10 images. (a) The solid curve shows, on average, how the number of representation updates increases over $50$ image classification tasks, with the shaded area showing one standard deviation. The dashed line represents linear growth. 
  (b) Each box plot shows the distribution of $0$-$1$ errors of the produced predictors when evaluated on held-out data from the CIFAR-10 test set in one of the $10$ independent trials.}
  \label{fig:cifar-exp}
\end{figure}

\subsection{Semi-synthetic experiments with CIFAR-10 Images}

So far, we have considered relatively simple settings involving linear synthetic data and the MNIST image dataset. To further address the second question (whether the learner requires only a small number of representation updates in more complex settings) and to assess the practicality of our algorithm, we now turn to a more realistic setting with CIFAR-10 images and deep convolutional representations.

\paragraph{Binary classification with CIFAR-10 images: experimental setup.} We evaluated the algorithm on images from CIFAR-10 \citep{krizhevsky2009learning}. As in the MNIST experiment, each task is formulated as a binary classification problem, with one CIFAR-10 class chosen as positive and five others as negative. We consider representations $h(x)$ given by a slightly modified ResNet-18 architecture with over $11$ million parameters, which outputs a $256$-dimensional feature vector (see Appendix~\ref{app:implementation-details} for details). For each task $t$, we again consider linear prediction layers parameterized by $w_t \in \RR^{k}$ such that predictions are in the form of $\mathds{1}\cbr{\inner{w_t}{h(x)} \ge 0}$, with $k = 256$ in this setting.

We consider the $0$-$1$ loss with a target absolute error of $\epsilon = 0.15$. The learner has access to the CIFAR-10 training set and uses binary cross-entropy as a surrogate loss during training. To make the experimental setup more realistic, we introduce minor practical modifications to the algorithm. We reserve a validation split from the CIFAR-10 training set consisting of $100$ images per CIFAR-10 class. For few-shot property tests, the learner draws $800$ samples from each of the positive and negative classes to train a linear prediction layer and evaluates its error on the validation split ($100$ validation images from each of the positive and negative classes). The test succeeds if the validation error is at most $\frac{2}{3}\epsilon$. If the test fails, the learner then draws $4900$ samples per class from the training split and reuses stored data in the memory to update its representation via multi-task ERM.

\paragraph{Results.} We ran $10$ trials with $T = 50$ tasks. On average, the learner updated it representation $8.0$ times with a standard deviation of $1.18$, and Figure~\ref{fig:cifar10-num-updates} shows that the number of updates grew sublinearly with the number of tasks. We also evaluated the produced predictors on data from the CIFAR-10 test set. Figure~\ref{fig:cifar10-errors} shows that the errors across tasks and trials were consistently below $\epsilon = 0.15$. These results confirm that even for more expressive function classes based on deep convolutional architectures, our algorithm can perform effective lifelong representation learning while using a modest number of representation updates via multi-task ERM. In addition, these results highlight that our algorithm can be implemented and run in a practical setup using standard architectures and training procedures.

\section{Discussion and future work }
\label{sec:discussion}

In Section~\ref{sec:task-eluder}, we introduced the notion of $\epsilon$-independence (Definition~\ref{def:eps-independence}). We now examine an alternative definition of $\epsilon$-independence.
\begin{definition}[$\epsilon$-pointwise-independence]
\label{def:eps-independence-alternative}
Given the setting of Definition~\ref{def:eps-independence},
for any representation $h \in \Hcal$, we say that $(h, f_{n})$ is
{\em $\epsilon$-independent} of $\cbr{(h, f_1), \ldots, (h, f_{n-1})}$ with respect to $(\Hcal, \Fcal)$ if there exist $h' \in \Hcal$ and $f'_1, \ldots, f'_{n-1} \in \Fcal$ such that
\begin{align}
    \label{eqn:max-eps-independence}
    \max_{i \in [n-1]} \EE_{P_{f_i \circ h}} \sbr{ \ell((f'_i \circ h')(x), y) - \ell((f_i \circ h)(x), y) }\le \epsilon, 
\end{align}
but for any $f'_{n} \in \Fcal$, $\EE_{P_{f_n \circ h}} \sbr{\ell((f'_{n} \circ h')(x), y) - \ell((f_{n} \circ h)(x), y) } > \frac{\epsilon}{2}$.
\end{definition}
At first glance, this seems to be a more natural definition. Eq.~\eqref{eqn:max-eps-independence} requires that $(h', f'_i)$'s incur low risk on \textit{each} of the earlier tasks induced by $(h, f_i)$'s, which seems well aligned with the lifelong learning objective. In addition, this condition is weaker than Eq.~\eqref{eqn:eps-independence-sum} in Definition~\ref{def:eps-independence}, i.e.,
\[
\sum_{i=1}^{n-1} \EE_{P_{f_i \circ h}} \sbr{ \ell((f'_i \circ h')(x), y) - \ell((f_i \circ h)(x), y) }\le \epsilon,
\]
and so, following a similar analysis, this definition would suggest an improved sample complexity bound.

However, we now provide a negative result which shows that the condition in Eq.~\eqref{eqn:max-eps-independence} can, in fact, be insufficient.
Consider noiseless, binary linear classification tasks with a low-dimensional linear representation. Let $\Hcal = \{x \mapsto B^\top x: B \in \RR^{d \times k}, B^\top B = I\}$, $\Fcal= \cbr{z \mapsto w^\top z: w \in \RR^k, \nbr{w} = 1}$, and $\ell$ be the $0$-$1$ loss. Suppose $P_X = \Ncal (0, I)$. Then, for unit vectors $u, v \in \RR^d$,
\[
\Pr_{x \sim P_X} \sbr{\sign(u^\top x) \neq \sign(v^\top x)} = \frac{\theta(u, v)}{\pi},
\]
where $\theta(\cdot, \cdot)$ denotes the angle between two vectors~\cite[e.g., ][]{long1995sample}.

\begin{proposition}
    \label{prop:infinite-eps-ind-sequence}
    In the noiseless linear classification setting described above, suppose $d > k$. There exists an infinite sequence of $\epsilon$-independent tasks under Definition~\ref{def:eps-independence-alternative}.
\end{proposition}
\noindent See Appendix~\ref{app:proof-task-eluder-alternative} for its proof. 
Despite this negative result, it may be possible to formulate intermediate notions of agreement on prior tasks that lie between the aggregate condition in Eq.~\eqref{eqn:eps-independence-sum} and the pointwise condition in Eq.~\eqref{eqn:max-eps-independence}, which could lead to improved sample complexity bounds. Understanding how to ensure that such conditions hold, both in terms of the number of samples needed and algorithmic techniques, would be an interesting direction for future study, with potential connections to collaborative PAC learning \citep{blum2017collaborative}. At the same time, it is important to also understand whether such definitions admit meaningful bounds on their induced notions of task-eluder dimension for natural function classes. 

Beyond this, several other open questions remain. In this work, we assumed a well-specified model (Assumpton~\ref{assum:well-specified}). While this relaxes the noiseless, realizable assumption studied in prior work, it remains to be seen if it can be further weakened. Other future directions include studying whether the learner can infer noise levels from other tasks under additional structural assumptions, and bounding the task-eluder dimension for more general function classes.

\section*{Acknowledgments}
This work was partially supported by NSF CAREER Award CCF 2238876. CZ acknowledges funding support by NSF CAREER Award IIS-2440266. 
 
\newpage
\putbib

\newpage
\appendix

\section*{Outline of the Appendix}

\startcontents[sections]
\printcontents[sections]{l}{1}{\setcounter{tocdepth}{2}}
\newpage

\newpage
\section{Supplementary material for Section~\ref{sec:algorithm}}
\subsection{Hardness of few-shot property testing}
\label{app:hardness-property-test}

\begin{proposition}[Formal version of Proposition~\ref{prop:property-test-neg}]
\label{prop:property-test-neg-formal}
Fix a subspace $U$ of $\RR^d$ with dimension 
$r \leq  d / 2$.
Consider two classes of functions, $\Gcal := \cbr{x \mapsto \inner{x}{w}: \nbr{w} \le 1}$ and 
$\Gcal_0 := \{x \mapsto \inner{x}{w}: \nbr{w} \le 1, w \in U\} \subset \Gcal$, where the latter is $\Gcal$ restricted to the subspace $U$ (equiv. to having a low-dimensional linear representation fixed). For any distribution $P$ and $\Gcal' \subseteq \Gcal$, let $\kappa_{P}(\Gcal') := \inf_{g \in \Gcal'} \EE_{(x,y) \sim P} \sbr{(g(x) - y)^2}$ denote the risk of the best function in $\Gcal'$.

Consider the following two hypotheses:
\[
H_0 = \cbr{P \in \Delta( \RR^d \times \RR ): \kappa_P(\Gcal_0) =\kappa_P(\Gcal)}, \ \text{and} \ H_1 = \cbr{P \in \Delta( \RR^d \times \RR ): \kappa_P(\Gcal_0) > \kappa_P(\Gcal) + 0.9}.
\]
Then, there exists some constant $c$ such that, with $n \le c\sqrt{d}$ examples, 
\[
\inf_{\psi: (\RR^d \times \RR)^n \to \cbr{0,1}}
\sup_{b \in \cbr{0,1}}
\sup_{P \in H_b}
\Pr_{(x_i, y_i)_{i=1}^n \sim P^n}
\rbr{
\psi( (x_i, y_i)_{i=1}^n ) \neq b
}
\geq \frac 1 3. 
\]

\end{proposition}

\begin{proof}
Let $U^\perp$ denote the orthogonal complement of $U$, and $P_{U^\perp}$ denote the orthogonal projection matrix onto $U^\perp$. 
Consider two distributions, $Q_0$ and $Q_1$, over $(\RR^d \times \RR)^n$:

\begin{itemize}[leftmargin=*]
    \item $Q_0$: each example is drawn i.i.d.\ from $P_0$.
    Here, we define $P_0$ such that $(x,y) \sim P_0$ iff $x \sim \Ncal(0, I_d)$, $y \sim \Ncal(0, 1)$, and $x, y$ are independent. We have $\kappa_{P_0}(\Gcal) = \kappa_{P_0}(\Gcal_0) = \sigma^2 = 1$; that is, $P_0 \in H_0$.

    \item $Q_1$: this is a mixture of distributions
    $P_1(w_*)^n$, 
    where $w_*$ is drawn from the uniform distribution over $\cbr{z \in U^\perp: \nbr{z} = \sqrt{0.99}}$.
    Herein, $P_1(w_*)$ is defined such that 
    $x \sim N(0, I_d)$, $y = \inner{w_*}{x} + \eta$, 
    $\eta \sim \Ncal(0, 0.01)$, and 
    $x, \eta$ are independent. 
    Then, for any $w_*$, we have $\kappa_{P_1(w_*)}(\Gcal) = 0.01$ and $\kappa_{P_1(w_*)}(\Gcal_0) = 0.01 + \nbr{P_{U^\perp} w_*}^2 = 1$. Therefore, for any $w_*$, $\kappa_{P_1(w_*)}(\Gcal_0) > \kappa_{P_1(w_*)}(\Gcal) + 0.9$, and $P_1(w_*) \in H_1$.
\end{itemize}

To prove Proposition~\ref{prop:property-test-neg-formal}, it suffices to show that
\begin{equation}
\inf_{\psi: (\RR^d \times \RR)^n \to \cbr{0,1}}
\sup_{b \in \cbr{0,1}}
\Pr_{(x_i, y_i)_{i=1}^n \sim Q_b}
\rbr{
\psi( (x_i, y_i)_{i=1}^n ) \neq b
}
\geq \frac 1 3. 
\label{eqn:hyp-test-original}
\end{equation}
To see this, we note that for any $b \in \cbr{0,1}$, 
\[
\sup_{P \in H_b}
\Pr_{(x_i, y_i)_{i=1}^n \sim P^n}
\rbr{
\psi( (x_i, y_i)_{i=1}^n ) \neq b
}
\geq 
\Pr_{(x_i, y_i)_{i=1}^n \sim Q_b}
\rbr{
\psi( (x_i, y_i)_{i=1}^n ) \neq b
},
\]
since $Q_b$ are mixtures of distributions of $P^n$, for $P \in H_b$.

\vspace{15pt}
\noindent To establish this, we reduce our hypothesis testing problem from
a known hypothesis testing problem due to \cite{kong2018estimating}, which is known to require a large number of samples:

\begin{proposition}[\citealt{kong2018estimating}, Proposition 2; see also \citealt{verzelen2010goodness}, Proposition 4.2]
\label{prop:kv18}
Consider two probability distributions in $\Delta((\RR^{d'} \times \RR)^n)$ over random variables $(x_i', y_i')_{i=1}^n$:
\begin{enumerate}[leftmargin=*,itemsep=0pt,topsep=2pt]
    \item Pure noise $Q_0'$: each $(x_i', y_i')$ is drawn independently such that $x_i' \sim \Ncal(0, I_{d'})$ and $y'_i \sim \Ncal(0,1)$, with $x_i'$ independent of $y'_i$;

    \item Little noise $Q_1'$: let $w_*'$ be a random vector drawn uniformly from the sphere in $\RR^{d'-1}$ with radius $\sqrt{0.99}$. Conditioned on $w_*'$, each $(x_i', y_i')$ is drawn independently such that $x'_i \sim \Ncal(0, I_{d'})$ and $y'_i = \inner{w_*'}{x'_i} + \eta_i$, where $\eta_i \sim \Ncal(0, 0.01)$.
\end{enumerate}
There is a universal constant $c$ such that 
when $n \leq c\sqrt{d'}$, 
\begin{equation}
\inf_{\psi': (\RR^{d'} \times \RR)^n \to \cbr{0,1}}
\sup_{b \in \cbr{0,1}}
\Pr_{(x_i', y_i')_{i=1}^n \sim Q_b'}
\rbr{
\psi( (x_i', y_i')_{i=1}^n ) \neq b
}
\geq \frac 1 3.
\label{eqn:hyp-test-kv18}
\end{equation}
\end{proposition}

We now provide the reduction from~\cite{kong2018estimating}'s task to ours.
Pick an orthogonal matrix $B = [B_\parallel \ B_\perp] \in R^{d \times d}$ whose first $r$ columns span $U$ and last $d'$ columns span $U^\perp$. For any $(x_i', y_i')$, consider 
\[
x_i = B_\perp x_i' + B_\parallel z_i,
\]
where $z_i \sim \Ncal(0, I_r)$ is independent of $x_i'$, and $y_i = y_i'$.
It is easy to verify that $x_i \sim \Ncal(0, I_d)$. Consider two cases:
\begin{itemize}[leftmargin=*,itemsep=0pt,topsep=2pt]
    \item Under $Q'_0$, $y_i' \sim \Ncal(0, 1)$ and so $(x_i,y_i)_{i=1}^n$ follows the distribution $Q_0$.

    \item Under $Q'_1$, $y_i' = \inner{x_i'}{w'_*} + \eta_i$, where $w'_*$ is from the uniform distribution over $\mathbb{S}^{d'}(0, \sqrt{0.99})$. Set $w_* = B_\perp w'_*$. Then, it is easy to verify that $y = \inner{x_i}{w_*} + \eta_i$, and $w_*$ is drawn from the uniform distribution over $\cbr{z \in U^\perp: \nbr{z} = \sqrt{0.99}}$.
    In other words, $(x_i, y_i)_{i=1}^n$ follows the same distribution as $Q_1$.

\end{itemize}

\noindent Given any test 
$\psi: (\RR^d \times \RR)^n \to \cbr{0,1}$ that tries to distinguish $Q_0$ and $Q_1$,
we now construct a test $\psi': (\RR^{d'} \times \RR)^n \to \cbr{0,1}$ that distinguishes $Q'_0$ and $Q'_1$ with the same worst-case error rate. 
Specifically, define
$\psi'( (x_i', y_i')_{i=1}^n )
= 
\psi( (x_i, y_i)_{i=1}^n )
$.
To see why they have the same worst-case error rates, we note that for any $b \in \cbr{0,1}$, 
\[
\Pr_{(x_i', y_i')_{i=1}^n \sim Q_b'} (\psi'((x_i', y_i')_{i=1}^n) \neq b)
= 
\Pr_{(x_i, y_i)_{i=1}^n \sim Q_b} (\psi((x_i, y_i)_{i=1}^n) \neq b),
\]
due to the construction of $(x_i,y_i)_{i=1}^n$ based on $(x_i',y_i')_{i=1}^n$ above.
Thus, Eq.~\eqref{eqn:hyp-test-original} follows from Eq.~\eqref{eqn:hyp-test-kv18}.

\end{proof}

\subsection{Proof of Proposition~\ref{prop:finite-eluder}}
\label{app:finite-task-eluder}

\begin{restate}{Proposition}{prop:finite-eluder}
For any $\epsilon \ge 0$,
\[
\dim(\Hcal, \Fcal, \epsilon) \le 2\min \rbr{\abr{\Hcal}, \abr{\Fcal}}.
\] 
\end{restate}

\begin{proof}
We first show that $\dim(\Hcal, \Fcal, \epsilon) \le 2\abr{\Hcal}$.
Let $\cbr{(h, f_i)}_{i=1}^n$ be any sequence of $\epsilon$-independent tasks. For each $j \le n$, let $\Hcal_j \subset \Hcal$ denote the set of possible witness representations. That is, $h' \in \Hcal_j$ if there exist some $f'_1, \ldots, f'_{j-1}$ such that
\begin{align}
    \sum_{i=1}^{j-1} \EE_{P_{f_i \circ h}} \sbr{ \ell((f'_i \circ h')(x), y) - \ell((f_i \circ h)(x), y) }\le \epsilon, 
\end{align}
but for any $f'_{j} \in \Fcal$, $\EE_{P_{f_j \circ h}} \sbr{\ell((f'_{j} \circ h')(x), y) - \ell((f_{j} \circ h)(x), y) } > \frac{\epsilon}{2}$.
For any $h' \in \Hcal$ and $j$, let 
\[
m_j(h') = \abr{\cbr{i: i \le j, h' \in \Hcal_i}}
\]
denote the number of times $h'$ appears in the set of witness representations up until task $j$. It suffices to show that for any $h'$, $m_n(h') \le 2$, i.e., each $h'$ can witness at most twice, as this implies that any sequence of $\epsilon$-independent tasks will have length at most $2|\Hcal|$. 

Suppose, for contradiction, $i_1 < i_2 < i_3$ are three indices in $[n]$ such that $h' \in \Hcal_{i_1} \cap \Hcal_{i_2} \cap \Hcal_{i_3}$. Since $h' \in \Hcal_{i_1} \cap \Hcal_{i_2}$, for every $f'_{i_1} \in \Fcal$, we have
\[
\EE_{P_{f_{i_1} \circ h}} \sbr{\ell((f'_{i_1} \circ h')(x), y) - \ell((f_{i_1} \circ h)(x), y) }  > \frac{\epsilon}{2};
\]
and for every $f'_{i_2} \in \Fcal$,
\[
\EE_{P_{f_{i_2} \circ h}} \sbr{\ell((f'_{i_2} \circ h')(x), y) - \ell((f_{i_2} \circ h)(x), y) }  > \frac{\epsilon}{2}.
\]
Furthermore, since $h' \in \Hcal_{i_3}$, there exist $f'_1, f'_2, \ldots, f'_{i_3}$ such that
\begin{align}
    \label{eqn:finite-task-eluder-H-i3}
    \sum_{i=1}^{i_3-1} \EE_{P_{f_i \circ h}} \sbr{ \ell((f'_i \circ h')(x), y) - \ell((f_i \circ h)(x), y) }\le \epsilon.
\end{align}
However, for any $f'_{i_1}, f'_{i_2} \in \Fcal$,
\[
\EE_{P_{f_{i_1} \circ h}} \sbr{\ell((f'_{i_1} \circ h')(x), y) - \ell((f_{i_1} \circ h)(x), y) } + \EE_{P_{f_{i_2} \circ h}} \sbr{\ell((f'_{i_2} \circ h')(x), y) - \ell((f_{i_2} \circ h)(x), y) } > \epsilon,
\]
which contradicts with Eq.~\eqref{eqn:finite-task-eluder-H-i3}. \\

We now turn to show $\dim(\Hcal, \Fcal, \epsilon) \leq 2 |\Fcal|$.
To this end, it suffices to show that for any sequence of $\epsilon$-independent tasks $(h, f_1), \ldots, (h, f_n)$, i.e., for any $i \in [n]$, $(h, f_i)$ is $\epsilon$-independent of $(h, f_1), \ldots, (h, f_{i-1})$, the number of appearances of any $f \in \Fcal$ in $\cbr{f_1, \ldots, f_n}$ can be at most 2. 
This will conclude the proof since this implies that any sequence of $\epsilon$-independent tasks will have length at most $2|\Fcal|$. 

Fix one such $\epsilon$-independent sequence and 
any $f \in \Fcal$. Assume towards contradiction that $i_1, i_2, i_3$ are three indices in $[n]$ such that $f_{i_1} = f_{i_2} = f_{i_3} = f$.
Since $(h, f_{i_3})$ is $\epsilon$-independent of $(h, f_1), \ldots, (h, f_{i_3-1})$, we have that there exists $h' \in \Hcal$ such that 
\begin{equation}
\sum_{i=1}^{i_3-1} \underbrace{\min_{f' \in \Fcal} \EE_{P_{f_i \circ h}}  \sbr{ \ell((f' \circ h')(x), y) - \ell((f_i \circ h)(x), y) }}_{ =: L_i(h')} \le \epsilon,
\label{eqn:in-task-error-small}
\end{equation}
but 
\begin{equation}
\underbrace{\min_{f' \in \Fcal} \EE_{P_{f_{i_3} \circ h}}  \sbr{ \ell((f' \circ h')(x), y) - \ell((f_i \circ h)(x), y) }}_{ =: L_{i_3}(h')} > \epsilon / 2,
\label{eqn:out-of-task-error-large}
\end{equation}
Since $L_{i_1}(h') = L_{i_2}(h') = L_{i_3}(h')$, Eq.~\eqref{eqn:in-task-error-small} implies that 
$L_{i_3}(h') \leq \frac \epsilon 2$, which contradicts with Eq.~\eqref{eqn:out-of-task-error-large}. 
\end{proof}

\begin{remark}
The constant $2$ in front of $2\min(\abr{\Hcal}, \abr{\Fcal})$ is due to out-of-task excess risk threshold  $\frac{\epsilon}{2}$ in the definition of $\epsilon$-independence of tasks (Definition~\ref{def:eps-independence}). If that threshold were $\frac{\epsilon}{C}$, we will obtain $\dim(\Hcal, \Fcal, \epsilon) \leq C \min(\abr{\Hcal}, \abr{\Fcal})$ here. In summary, our definition of task-eluder dimension is robust to the choice of the constant in the out-of-task excess risk threshold.
\end{remark}

\subsection{Proof of Proposition~\ref{prop:infinite-eps-ind-sequence}}
\label{app:proof-task-eluder-alternative}

For completeness, we restate below the alternative definition of $\epsilon$-independence ($\epsilon$-pointwise-independence), the setting of the negative example, and Proposition~\ref{prop:infinite-eps-ind-sequence}. \\

\begin{restate}
{Definition}{def:eps-independence-alternative}
Given the setting of Definition~\ref{def:eps-independence},
for any representation $h \in \Hcal$, we say that $(h, f_{n})$ is
{\em $\epsilon$-independent} of $\cbr{(h, f_1), \ldots, (h, f_{n-1})}$ with respect to $(\Hcal, \Fcal)$ if there exist $h' \in \Hcal$ and $f'_1, \ldots, f'_{n-1} \in \Fcal$ such that
\begin{align*}
    \max_{i \in [n-1]} \EE_{P_{f_i \circ h}} \sbr{ \ell((f'_i \circ h')(x), y) - \ell((f_i \circ h)(x), y) }\le \epsilon, 
\end{align*}
but for any $f'_{n} \in \Fcal$, $\EE_{P_{f_n \circ h}} \sbr{\ell((f'_{n} \circ h')(x), y) - \ell((f_{n} \circ h)(x), y) } > \frac{\epsilon}{2}$.
\end{restate}

\paragraph{Setting.} We consider noiseless, binary linear classification tasks under a shared low-dimensional linear representation. Let $\Hcal = \cbr{x \mapsto B^\top x: B \in \RR^{d \times k}, B^\top B = I}$, $\Fcal= \cbr{z \mapsto w^\top z: w \in \RR^k, \nbr{w} = 1}$, and let $\ell$ denote the $0$-$1$ loss. Suppose $P_X = \Ncal (0, I)$. Then, for unit vectors $u, v \in \RR^d$,
\[
\Pr_{x \sim P_X} \sbr{\sign(u^\top x) \neq \sign(v^\top x)} = \frac{\theta(u, v)}{\pi},
\]
where $\theta(\cdot, \cdot)$ denotes the angle between two vectors. \\

\begin{restate}{Proposition}{prop:infinite-eps-ind-sequence}
In the noiseless linear classification setting described above, suppose $d > k$. There exists an infinite sequence of $\epsilon$-independent tasks under Definition~\ref{def:eps-independence-alternative}.
\end{restate}

\begin{proof}[Proof of Proposition~\ref{prop:infinite-eps-ind-sequence}]
    Fix $B = [e_1 \ e_2 \ \ldots \ e_k] \in \RR^{d \times k}$, where $e_i$ denotes the $i$-th standard basis vector in $\RR^d$. Let $v = [1 \ 0\ \ldots\ 0] \in \RR^k$. Consider a sequence, $\cbr{(B, w_n)}_n$, where $w_n \equiv v$; that is, the same task is seen at every step. It suffices to show that for every $n$, $(B, w_n)$ is $\epsilon$-independent of $\cbr{(B, w_1), \ldots, (B, w_{n-1})}$. To this end, consider $B' = [s \ e_3 \ \ldots \ e_k \ e_{k+1}] \in \RR^{d \times k}$, where
    \[
    s = e_1 \cos\lambda + e_2 \sin \lambda, \qquad \lambda \in \left(\frac{\pi \epsilon}{2}, \pi \epsilon\right].
    \]
    Let $U = \linearspan(B')$.
    Observe that
    \begin{enumerate}[leftmargin=*]
        \item for each $i < n$, by choosing $f'_i$ to be $v$,
        \begin{align*}
            &\ \EE_{P_{f_i \circ h}} \sbr{ \ell((f'_i \circ h')(x), y) - \ell((f_i \circ h)(x), y) } \\
            =& \ \Pr_{x \sim P_X} \sbr{\sign(x^\top B'v) \neq \sign(x^\top Bv)}  \\
            = & \ \frac{\theta(Bv, B'v)}{\pi} = \frac{\lambda}{\pi} \le \epsilon;
        \end{align*}

        \item for any $f'_n$, 
        \begin{align*}
            &\ \EE_{P_{f_n \circ h}} \sbr{ \ell((f'_n \circ h')(x), y) - \ell((f_n \circ h)(x), y) } \\
            \ge & \ \frac{\theta(Bv, U)}{\pi} \\
            = & \ \frac{\theta(Bv, s)}{\pi} 
            = \frac{\lambda}{\pi} > \frac{\epsilon}{2},
        \end{align*}
        where $\theta(Bv, U)$ denotes the angle between $Bv$ and the subspace $U$.
    \end{enumerate}

 Therefore, by Definition~\ref{def:eps-independence-alternative}, $(B, v_n)$ is $\epsilon$-independent of its predecessors at every step $n$, which yields an infinite sequence.
\end{proof}

\newpage
\subsection{Pseudocode for lifelong representation learning with known task-eluder dimension}
\label{app:pseudocode-known-task-eluder}

Algorithm~\ref{alg:lifelong-known-task-eluder} provides pseudocode of our lifelong representation learning algorithm in the setting where the task-eluder dimension is known in advance.

\begin{algorithm}[h]
\SetNoFillComment
\SetCommentSty{commentfont}
\caption{Lifelong representation learning (known task-eluder dimension)}
\label{alg:lifelong-known-task-eluder}
\KwIn{$\Hcal$, $\Fcal$, target error $\epsilon$, $\dim(\Hcal, \Fcal, \epsilon)$, confidence $\delta$, number of tasks $T$, noise levels $(\kappa_t)_t$\;}

\vspace{2pt}
Initialize memory $\Mcal \gets \emptyset$ and $N \gets \dim(\Hcal, \Fcal, \epsilon)$\;

\vspace{2pt}
\For{task $t = 1$}{
    \vspace{2pt}
    Draw a sample $S_1$ of size $m_N$ from $\Pcal_1^{m_N}$, apply ERM to learn $\hat{h}$ and $\hat{f}_1$\, and output $\hat{f}_1 \circ \hat{h}$\;

    \vspace{2pt}
    Set $n \gets 1$ and $t_n \gets t$, and update the memory $\Mcal \gets \Mcal \cup \cbr{S_{t_n}}$\;
}

\vspace{5pt}
\For{tasks $t = 2, \ldots, T$}{
    \vspace{2pt}
    \tcp{Few-shot property test: check if $\hat{h}$ admits a hypothesis for current task $t$ with risk at most $\epsilon$} 
    Draw a sample $\tilde{S}_t$ of size $\tilde{m}$ from $\Pcal_t^{\tilde{m}}$, and apply ERM with current $\hat{h}$ to learn $\tilde{f}_t$\;

    \vspace{5pt}
    \If{$\widehat{\Lcal}_{\tilde{S}_t}(\tilde{f}_t \circ \hat{h}) \le \kappa_t + \frac{3}{4}\epsilon$}{
        \vspace{3pt}
        Output $\tilde{f}_t \circ \hat{h}$\; 
    }
    \Else{
        \vspace{2pt}
        
        \tcp{Multi-task ERM on a subset of tasks where few-shot property test failed (plus task 1)}

        \vspace{4pt}

        Set $n \gets n+1$\;
        Draw a sample $S_t$ of size $m_N$ from $\Pcal_t^{m_N}$\;
        Set $t_n \gets t$, and update the memory $\Mcal \gets \Mcal \cup \cbr{S_{t_n}}$\;
        \vspace{4pt}
        Apply ERM over the samples stored in the memory to learn
        \[
        \hat{h}, \check{f}^{(t)}_1, \ldots, \check{f}^{(t)}_{n} \gets \argmin_{\substack{h \in \Hcal \\ f_1, \ldots, f_{n} \in \Fcal}} \frac{1}{n} \sum_{i=1}^{n} \widehat{\Lcal}_{S_{t_i}} (f_i, h);
        \]

        Set $\hat{f}_t \gets \check{f}^{(t)}_{n}$, update $\hat{h}$, and output $\hat{f}_t \circ \hat{h}$\;
    }
}
\end{algorithm}

\newpage
\section{Supplementary material for Section~\ref{sec:theoretical-guarantees}}

\subsection{Proof of Theorem~\ref{thm:main}} \label{app:proofs-main}
We first state a more precise version of Theorem~\ref{thm:main} with explicit constants. 
\begin{theorem}[Restatement of Theorem~\ref{thm:main}]
\label{thm:restate-main}
Let $\Xi = \dim(\Hcal, \Fcal, \epsilon) < \infty$. Suppose $\Ccal \rbr{\Hcal, \frac{\epsilon}{256\Xi}} < \infty$ and $\Ccal \rbr{\Fcal, \frac{\epsilon}{256\Xi}} < \infty$. 
In Algorithm~\ref{alg:lifelong-nonparametric}, for each $N$, set
\[
m_N = \frac{256N}{\epsilon^2} \sbr{\log \Ccal \rbr{\Hcal, \frac{\epsilon}{64N}} + N \log \Ccal\rbr{\Fcal, \frac{\epsilon}{64N}} + \log \frac{16\log T \cdot \sum_{i=0}^{\log N} \binom{T}{2^i} }{\delta}} + \frac{64}{\epsilon^2},
\]
and set
\[
\tilde{m} = \frac{1024}{\epsilon^2} \sbr{\log \Ccal\rbr{\Fcal, \frac{\epsilon}{128}} + \log \frac{8T}{\delta}} + \frac{256}{\epsilon^2}.
\]

\vspace{10pt}
\noindent Then, with probability at least $1 - \delta$,
\begin{itemize}[leftmargin=*,itemsep=0pt,topsep=2pt]
    \item[-] For every task, algorithm~\ref{alg:lifelong-nonparametric} outputs a predictor with excess risk at most $\epsilon$;

    \item[-] Algorithm~\ref{alg:lifelong-nonparametric} performs multi-task ERM at most $2\Xi$ times;

    \item[-]  The sample complexity of Algorithm~\ref{alg:lifelong-nonparametric} is upper bounded by
    \begin{align*}
     \tilde{\Ocal} \bigg(\frac{T}{\epsilon^2} \log \Ccal\rbr{\Fcal, \frac{\epsilon}{128}} + \frac{\Xi^2}{\epsilon^2} \sbr{\log \Ccal\rbr{\Hcal, \frac{\epsilon}{128\Xi}} + \Xi \log \Ccal\rbr{\Fcal, \frac{\epsilon}{128\Xi}}}\bigg).
    \end{align*}
    In addition, the size of the memory buffer it requires is at most 
    \[
    \tilde{\Ocal} \rbr{\frac{\Xi^2}{\epsilon^2} \sbr{\log \Ccal\rbr{\Hcal, \frac{\epsilon}{128\Xi}} + \Xi \log \Ccal\rbr{\Fcal, \frac{\epsilon}{128\Xi}}}}.
    \]
\end{itemize}

\end{theorem}
\begin{proof}%
We begin by introducing some additional notation.
For each task $t$, let $\hat{h}_t$ denote the maintained representation at the beginning of task $t$. 
Let $n_t$ and $N_t$ denote the values of $n$ and $N$ at the beginning of task $t$, respectively. Denote by $N_{T+1}$ the value of $N$ at the end of task $T$.

We consider a ``sample tape'' model \citep[e.g.,][]{lattimore2020bandit}, where for each $f \circ h$, there is a stack of i.i.d.\ samples drawn from $P_{f \circ h}$ before the first task begins. When the learner draws a sample of size $m$ from $P_{f \circ h}$, it receives the next $m$ entries from the corresponding tape, denoted by $S_{f \circ h}(m)$. To avoid clutter, we use $S_t(m)$ to refer to a sample of size $m$ from $\Pcal_t = P_{f_t^* \circ h^*}$.

\paragraph{Clean event.} Consider the following events. Let 
\[
\Acal := \cbr{\forall t \in \cbr{2, \ldots, T}, \ \big |\Lcal_{\Pcal_t}(f \circ \hat{h}_t) -  \widehat{\Lcal}_{S_t(\tilde{m})}(f \circ \hat{h}_t) \big| \le \frac{\epsilon}{4}, \ \forall f \in \Fcal },
\]
and for each $N \in \cbr{2^0, 2^1, 2^2, \ldots,  2^{\lceil \log T \rceil}}$, let
\begin{align*}
    \Bcal_N := \Bigg \{\forall n \in [N], \ & \forall (t_1, \ldots, t_n) \in \binom{[T]}{n}, \\  & \abr{\sum_{i=1}^{n} \Lcal_{\Pcal_{t_i}}(f_i \circ h) - \sum_{i=1}^{n} \widehat{\Lcal}_{S_{t_i}(m_N)}(f_i \circ h)} \le \epsilon, \ \forall (h, f_1, \ldots, f_n) \Bigg\}.
\end{align*}
We now define the following notion of a {\em clean event},
\[
\Ecal = \rbr{ \bigcap_{i=0}^{\lceil \log T \rceil} \Bcal_{2^i} } \bigcap \Acal.
\]
Intuitively, $\Acal$ is the event that all samples used for few-shot property tests are $\frac{\epsilon}{4}$-representative. $\Bcal_N$ is the event that every subset of tasks of size at most $N$ satisfies the condition of a ``good event'' as described in \citep{baxter2000model} for multi-task ERM (see Lemma~\ref{lem:baxter-multi-task-detailed}). 

We claim that the clean event happens with high probability, i.e., $\Pr (\Ecal) \ge 1 - \delta$. To see this, we first examine $\Pr (\Acal)$. For any $t = 2, \ldots, T$, applying Lemma~\ref{lem:baxter-multi-task-detailed} with one task and a singleton representation class $\{\hat{h}_t\}$---that is, $\log \Ccal(\{\hat{h}_t\}, \frac{\epsilon_0}{32}) = 0$---we have
\[
\Pr \rbr{\forall f \in \Fcal, \ \abr{\Lcal_{\Pcal_t}(f \circ \hat{h}_t) -  \widehat{\Lcal}_{\tilde{S}_t}(f \circ \hat{h}_t)} \le \frac{\epsilon}{4}} \ge 1 - \frac{\delta}{2T}.
\]
It then follows that $\Pr(\Acal) \ge 1 - \frac{\delta}{2}$ by the union bound. 
Now, for each $\Bcal_N$, fix any $n \le N$ and $(t_1, \ldots, t_n)$ of $\binom{[T]}{n}$. By Lemma~\ref{lem:baxter-multi-task-detailed} and the definition of $m_N$ in Theorem~\ref{thm:restate-main},
\[
\Pr \rbr{ \exists (h, f_1, \ldots, f_n), \ \abr{\sum_{i=1}^{n} \Lcal_{P_{t_i}}(f_i \circ h) - \sum_{i=1}^{n} \widehat{\Lcal}_{S_{t_i}(m_N)}(f_i \circ h)} > \frac{\epsilon}{2} } \le \frac{\delta}{ 4\log T \cdot \sum_{j=1}^{\log N} \binom{T }{2^j} }.
\]
It then follows by the union bound that $\Pr (\Bcal_N) \ge 1 - \frac{\delta}{4\log T}$ for every $N$ and again by the union bound that
\[
\Pr \rbr{\bigcap_{i=0}^{\lceil \log T \rceil} \Bcal_{2^i}} \ge 1 - \frac{\delta}{2}.
\]

\textbf{Correctness of Algorithm~\ref{alg:lifelong-nonparametric}.} We now show that, under the clean event $\Ecal$, Algorithm~\ref{alg:lifelong-nonparametric} outputs a predictor with excess risk at most $\epsilon$ for every task. 

For task $1$, when $\Bcal_1$ happens, for any $h \in \Hcal$ and $f \in \Fcal$,
$\abr{\Lcal_{\Pcal_1}(f \circ h) -  \widehat{\Lcal}_{S_1}(f \circ h)} \le \frac{\epsilon}{2}$,
and so the solution to single-task ERM has excess risk $\le \epsilon$.

For each subsequent task $t = 2, \ldots, T$, the algorithm draws an i.i.d.\ sample $\tilde{S}_t$ of size $\tilde{m}$ to perform a few-shot property test. When $\Acal$ happens, for any $f \in \Fcal$,
$\abr{\Lcal_{\Pcal_t}(f \circ \hat{h}_t) -  \widehat{\Lcal}_{\tilde{S}_t}(f \circ \hat{h}_t)} \le \frac{\epsilon}{4}$.
Let $\tilde{f}_t$ denote the ERM solution. Then, if $\widehat{\Lcal}_{\tilde{S}_t}(\tilde{f}_t \circ \hat{h}_t) \le \kappa_t + \frac{3\epsilon}{4}$, the true risk $\Lcal_{\Pcal_t}(\tilde{f}_t \circ \hat{h}_t) \le \kappa_t + \epsilon$, where we recall that $\kappa_t$ is the risk of $f_t^* \circ h^*$. The algorithm can safely move on to the next task.

Otherwise, the algorithm performs multi-task ERM. Consider two cases: 
\begin{enumerate}[leftmargin=*]
    \item ($n_t < N_t$) Let $t_1, \ldots, t_{n_t}$ denote the past tasks for which data are {\em currently} stored in memory, and let $t_{n_t + 1} = t$.
    Given i.i.d.\ samples of size $m_{N_t}$ from each of the $(n_t + 1)$ tasks, when $\Bcal_{N_t}$ happens, 
    \[
    \abr{\sum_{i=1}^{n_t + 1} \Lcal_{P_{t_i}}(f_i \circ h) - \sum_{i=1}^{n_t + 1} \widehat{\Lcal}_{S_{t_i}}(f_i \circ h)} \le \frac{\epsilon}{2}.
    \]
    Let $(\hat{h}, \check{f}^{(t)}_1, \ldots, \check{f}^{(t)}_{n_t + 1})$ denote the solution to multi-task ERM. It follows that
    \begin{align*}
         & \hspace{15pt} \sum_{i=1}^{n_t + 1} \Lcal_{P_{t_i}}(\check{f}^{(t)}_i \circ \hat{h}) -  \sum_{i=1}^{n_t + 1}  \Lcal_{P_{t_i}}(f_{t_i}^* \circ h^*) \\
        & \le \rbr{\sum_{i=1}^{n_t + 1} \Lcal_{P_{t_i}}(\check{f}^{(t)}_i \circ \hat{h}) - \sum_{i=1}^{n_t + 1} \widehat{\Lcal}_{S_{t_i}}(\check{f}^{(t)}_i \circ \hat{h})} + \underbrace{\rbr{\sum_{i=1}^{n_t + 1} \widehat{\Lcal}_{S_{t_i}}(\check{f}^{(t)}_i \circ \hat{h}) - \sum_{i=1}^{n_t + 1} \widehat{\Lcal}_{S_{t_i}}(f_{t_i}^* \circ h^*)}}_{\le 0} \\
        & \quad + \rbr{\sum_{i=1}^{n_t + 1} \widehat{\Lcal}_{S_{t_i}}(f_{t_i}^* \circ h^*) - \sum_{i=1}^{n_t + 1}  \Lcal_{P_{t_i}}(f_{t_i}^* \circ h^*)} \\
        & \le \epsilon.
    \end{align*}
    Since each summand in $\sum_{i=1}^{n_t + 1} \rbr{\Lcal_{P_{t_i}}(\check{f}^{(t)}_i \circ \hat{h}) - \Lcal_{P_{t_i}}(f_{t_i}^* \circ h^*)}$ is nonnegative, we have
    \[
     \Lcal_{P_{t}}(\check{f}^{(t)}_{n_t+1} \circ \hat{h}) -  \Lcal_{P_{t}}(f_{t}^* \circ h^*) \le \epsilon.
    \]

    \item ($n_t = N_t$) In this case, we have $n_{t+1} = 1$ and $N_{t+1} = 2N_t$. The memory is cleared. The algorithm then performs single-task ERM with an i.i.d.\ sample of size $m_{N_{t+1}}$. Under the event $\Bcal_{N_{t+1}}$, the ERM solution has excess risk at most $\epsilon$.
\end{enumerate}

\paragraph{Bounding the number of times Algorithm~\ref{alg:lifelong-nonparametric} performs multi-task ERM.} 

We now show that when $\Ecal$ happens, $N_{T+1} < 2\Xi$.

Assume towards contradiction that $N_{T+1} \ge 2\Xi$. 
Now consider the first task $t$ such that $N_{t+1} \ge 2\Xi$.
At task $t$, line 11 must be reached, and we have $n_t = N_t \ge \Xi$ at the beginning of task $t$. 
For this task, we have
\[
\min_{f \in \Fcal} \widehat{\Lcal}_{\tilde{S}_t}(f \circ \hat{h}_t) > \kappa_t + \frac{3}{4}\epsilon
\implies 
\min_{f \in \Fcal} \Lcal_{\Pcal_t}(f \circ \hat{h}_t) - \Lcal_{\Pcal_t}(f_t^* \circ h^*) > \frac{1}{2}\epsilon.
\]
Abbreviate $n_t$ as $n$. 
In addition, we have that for the tasks $t_1, \ldots, t_{n}$ in the current memory, $(\hat{h}_t,  \check{f}_1^{(t_n)}, \ldots, \check{f}_{n}^{(t_n)})$ satisfies 
\[
\sum_{i=1}^{n} 
\Lcal_{\Pcal_{t_i}}(\check{f}_{i}^{(t_n)} \circ \hat{h}) - \sum_{i=1}^{n} \Lcal_{\Pcal_{t_i}}(f_{t_i}^* \circ h^*) \le \epsilon.
\]
This implies that $(h^*, f_{t}^*)$ is $\epsilon$-independent of $\cbr{(h^*, f_{t_i}^*)}_{i=1}^{n}$. This argument can be extended for each task in the current memory---each $(h^*, f_{t_i}^*)$ must be $\epsilon$-independent of its predecessors. Since the task-eluder dimension is bounded by $\Xi$, this implies that $n_t +1 \leq \Xi$, which contradicts with the assumption that $n_t \ge \Xi$. Hence, we have $N_{T+1} < 2\Xi$. 

\paragraph{Sample and space complexity.} We have shown that $N_{T+1} < 2\Xi$. Let $b = \log_2 N_{T+1}$. The total number of samples used by Algorithm~\ref{alg:lifelong-nonparametric} is upper bounded by
\[
(\star) := \tilde{m} T + \sum_{i=0}^{b} 2^i \cdot m_{2^i},
\]
where the first term is from few-shot property testing, and the second term is from samples collected and saved in memory for multi-task ERM. For each value of $N = 2^i$, $i = 0, 1, \ldots, b$, samples are drawn from at most $N$ tasks with $m_N$ samples per task. We focus on the second term:
\hspace{-100pt}
\begin{align*}
    & \quad \ \sum_{i=0}^{b} 2^i \cdot m_{2^i} \\
    & = \sum_{i=0}^{b} 2^i \rbr{ \frac{256 \cdot 2^i}{\epsilon^2} \sbr{\log \Ccal \rbr{\Hcal, \frac{\epsilon}{64 \cdot 2^i}} + 2^i \log \Ccal\rbr{\Fcal, \frac{\epsilon}{64 \cdot 2^i}} + \log \frac{16\log T \cdot \sum_{j=0}^{i} \binom{T}{2^j} }{\delta}} + \frac{64}{\epsilon^2}} \\
    & \stackrel{\mathrm{(a)}}{\le} \sum_{i=0}^{b} 2^i \rbr{ \frac{256 \cdot 2^i}{\epsilon^2} \sbr{\log \Ccal \rbr{\Hcal, \frac{\epsilon}{64 \cdot 2^i}} + 2^i \sbr{\log \Ccal\rbr{\Fcal, \frac{\epsilon}{64 \cdot 2^i}} + \log (eT)} + \log \frac{16 \log T}{\delta}} + \frac{64}{\epsilon^2}} \\
    & \stackrel{\mathrm{(b)}}{\le} \sum_{i=0}^{b} 4^i \cdot \frac{256}{\epsilon^2} \sbr{\log \Ccal \rbr{\Hcal, \frac{\epsilon}{64 N_{T+1}}} + \log \frac{16 \log T}{\delta}} \\
    & \hspace{150pt} + \sum_{i=0}^b 8^i \cdot \frac{256}{\epsilon^2} \sbr{ \log \Ccal\rbr{\Fcal, \frac{\epsilon}{64N_{T+1}}} + \log(eT)} + \sum_{i=0}^b 2^i \cdot \frac{64}{\epsilon^2} \\
    & \stackrel{\mathrm{(c)}}{\lesssim} \frac{N_{T+1}^2}{\epsilon^2} \sbr{\log \Ccal \rbr{\Hcal, \frac{\epsilon}{64 N_{T+1}}} + N_{T+1} \sbr{\log \Ccal\rbr{\Fcal, \frac{\epsilon}{64 N_{T+1}}} + \log(eT)} + \log \frac{16 \log T}{\delta}} + \frac{N_{T+1}}{\epsilon^2},
\end{align*}
where (a) follows because 
\[
\log \rbr{\sum_{j=0}^i \binom{T}{2^j}} \le 2^i \log \rbr{\frac{eT}{2^i}} \le 2^i \log \rbr{eT};
\]

(b) follows by algebra and the observation that covering numbers increase as scale decreases; and (c) uses the following inequality for bounding the sum of a finite geometric series, $\sum_{i=0}^b r^i \le r^{b+1}$ for $r \ge 2$. 

Since $N_{T+1} < 2\Xi$, we have
\begin{align*}
    (\star) \le \tilde{\Ocal} \rbr{\frac{T}{\epsilon^2} \log \Ccal\rbr{\Fcal, \frac{\epsilon}{128}}  + \frac{\Xi^2}{\epsilon^2} \sbr{\log \Ccal \rbr{\Hcal, \frac{\epsilon}{128\Xi}} + \Xi \log \Ccal\rbr{\Fcal, \frac{\epsilon}{128\Xi}}}}.
\end{align*}

Similarly, the size of the memory needed is non-decreasing. Since $N_{T+1} < 2\Xi$, the space complexity is upper bounded by
\begin{align*}
    N_{T+1} \cdot m_{N_{T+1}} 
    & \le \tilde{\Ocal} \rbr{\frac{\Xi^2}{\epsilon^2} \sbr{\log \Ccal \rbr{\Hcal, \frac{\epsilon}{128\Xi}} + \Xi \log \Ccal\rbr{\Fcal, \frac{\epsilon}{128\Xi}}}}. 
\end{align*}

\end{proof}

\subsection{Auxiliary lemma}
Here, we provide a more precise version of the guarantee for multi-task ERM from \citep{baxter2000model} with explicit constants. Lemma~\ref{lem:baxter-multi-task-detailed} follows from Corollary 19 and Theorem 6 thereof.
\begin{lemma}[\citealp{baxter2000model}]
\label{lem:baxter-multi-task-detailed}
Let $P_1, \ldots, P_n$ be the data distributions for $n$ tasks. Let $\Hcal_0$ be a class of representations and $\Fcal_0$ be a class of prediction layers. Let $\epsilon_0, \delta_0 \in (0, 1)$. Suppose for each task, an i.i.d.\ sample $S_i$ of size $m$ is drawn from $P_i^m$, where
\[
m \ge \frac{64}{n\epsilon_0^2} \sbr{\log \Ccal(\Hcal_0, \frac{\epsilon_0}{32}) + n\log \Ccal(\Fcal_0, \frac{\epsilon_0}{32}) + \log \frac{4}{\delta_0}} + \frac{16}{\epsilon_0^2}.
\]
Then, with probability at least $ 1- \delta_0$, for any $(h, f_1, \ldots, f_n)$,
\[
\abr{\frac{1}{n} \sum_{i=1}^n \Lcal_{P_i}(f_i \circ h) - \frac{1}{n} \sum_{i=1}^n \widehat{\Lcal}_{S_i}(f_i \circ h)} \le \epsilon_0.
\]
\end{lemma}

\newpage
\section{Supplementary material for Section~\ref{sec:examples}} \label{app:proofs-examples}

In Appendix~\ref{app:proofs-examples}, we first present a key lemma for bounding the task-eluder dimension and then provide proofs for the examples in Section~\ref{sec:examples}. Auxiliary lemmas are deferred to Appendix~\ref{app:auxiliary-lemmas}.

\subsection{Key lemma for bounding the task-eluder dimension}
\begin{lemma}
\label{lem:task-eluder-linear}
Let $\Xcal \subset \RR^d$ and $P_X$ be a distribution over $\Xcal$. 
Let $\Ycal \subseteq \RR$.
Consider $\Hcal=\{x \mapsto B^\top x: B \in \RR^{d \times k}, B^\top B = I_k\}$ and $\Fcal = \cbr{z \mapsto g(w^\top z): w \in \RR^k, \underline{b} \le \nbr{w} \le \overline{b}}$, where $0 \le \underline{b} \le \overline{b} \le 1$ and $g$ is a (possibly nonlinear) map from $\RR$ to $\RR$. 

For any $h \in \Hcal$ and $f \in \Fcal$, let $P_{f \circ h}(x,y) = P_X(x) P_{Y | X}(y | x; f \circ h)$. Let $\PP = \cbr{P_{f \circ h}: h \in \Hcal, f \in \Fcal}$, and $\ell: \RR \times \Ycal \rightarrow [0,1]$ be a loss function. 

Fix $p \in [1,2]$. For any representations $h, h' \in \Hcal$, let $B, B' \in \RR^{d \times k}$ be the corresponding matrices. Similarly, for any $f, f' \in \Fcal$, let $w, w' \in \RR^k$ be the corresponding vectors. 
Suppose
\begin{align}
    \label{eqn:task-eluder-linear-assumption}
    \nbr{B'w' - Bw}^p \ \lesssim \ \EE_{P_{f \circ h}} \sbr{\ell((f' \circ h')(x), y) - \ell((f \circ h)(x), y)}\ \lesssim \ \nbr{B'w' - Bw}^p.
\end{align}
Then, for any $\epsilon \in (0,1)$,
\[
\dim_{\PP, \ell} \rbr{\Hcal, \Fcal, \epsilon} \lesssim k \log \frac{1}{\epsilon}.
\]
\end{lemma}

\begin{proof}[Proof of Lemma~\ref{lem:task-eluder-linear}]
Let $\cbr{(h, f_i)}_{i=1}^\tau$ be any sequence of tasks such that, for each $n \le \tau$, $(h, f_n)$ is $\epsilon$-independent of $\cbr{(h,f_i)}_{i=1}^{n-1}$.
We show that $\tau \lesssim k \log \frac{1}{\epsilon}$.

To this end, let $B \in \RR^{d \times k}$ be the semi-orthogonal matrix associated with $h$, and for each $i \in [\tau]$, let $w_i \in \RR^k$ be the vector associated with $f_i$.
For $n \in [\tau]$, let $W_n \in \RR^{k \times n}$ denote the matrix whose columns are $w_1, \ldots, w_n$, and let $V_n := W_{n} W_{n}^\top + \epsilon^{\frac{2}{p}} I$. 

To conclude the proof, it suffices to show that for each $n \in \cbr{2, \ldots, \tau}$, $\nbr{w_n}^2_{V_{n-1}^{-1}} \gtrsim \frac14$.
Indeed, by the elliptical potential lemma \citep[][see also Lemma~\ref{lem:elliptical-potential}]{abbasi2011improved},
\[
\sum_{n=1}^\tau \min \cbr{1, \nbr{w_n}_{{V}_{n-1}^{-1}}^2} \le 2k \log \rbr{1 + \frac{\tau}{k\epsilon^2}},
\]
and it follows from Lemma~\ref{lem:elliptical-potential-bound-t} \citep[see also][]{lattimore2020bandit} that $\tau \lesssim k \log \frac{1}{\epsilon}$. \\

For $n \in \cbr{2, \ldots, \tau}$, we now show $\nbr{w_n}^2_{V_{n-1}^{-1}} \gtrsim \frac14$. 
Since $(h,f_n)$ is $\epsilon$-independent of $\cbr{(h,f_i)}_{i=1}^{n-1}$, there exist $h', f'_1, \ldots, f'_{n-1}$ such that
$\sum_{i=1}^{n-1} \EE_{P_{f_i \circ h}}\sbr{\ell \rbr{f'_i \circ h'(x), y} - \ell \rbr{f_i \circ h(x), y}} \le \epsilon$,
but for any $f'_n$, $\EE_{P_{f_n \circ h}} [\ell (f'_n \circ h'(x), y) - \ell (f_n \circ h(x), y)] > \frac{\epsilon}{2}$.
Fix any such $h'$ and $f'_1, \ldots, f'_{n-1}$, and let $B'$ and $w'_1, \ldots, w'_{n-1}$ denote the matrix and vectors associated with these functions, respectively.

Let $\alpha^{\star} := \argmin_{\alpha \in \RR^{n-1}} \epsilon^{-\frac{2}{p}} \nbr{w_n - W_{n-1} \alpha}_2^2 + \nbr{\alpha}_2^2$.
By Lemma~\ref{lem:ridge-woodbury-identity}, we have
\begin{align}
\label{eqn:w_n_V_n_decomp}
w_n^\top V_{n-1}^{-1} w_n = \epsilon^{-\frac{2}{p}} \nbr{w_n - W_{n-1} \alpha^{\star}}_2^2 + \nbr{\alpha^{\star}}_2^2.
\end{align}
Consider the decomposition $w_n = W_{n-1} \alpha^{\star} + z$. We have
\begin{align}
\label{eqn:linear-task-eluder-key}
\epsilon^{\frac1p} \stackrel{\mathrm{(a)}}{\lesssim} \nbr{P_{B'}^\perp Bw_n} & \stackrel{\mathrm{(b)}}{\le} \nbr{P_{B'}^\perp BW_{n-1} \alpha^{\star}} + \nbr{P_{B'}^\perp Bz} 
\stackrel{\mathrm{(c)}}{\lesssim} \epsilon^{\frac1p} \nbr{\alpha^{\star}} + \nbr{z},
\end{align}
where $P_{B'}^\perp$ denotes the orthogonal projection onto the orthogonal complement of $\linearspan(B')$, and the inequalities are justified shortly.
Dividing $\epsilon^{\frac{1}{p}} > 0$ on both sides, it follows that either $\nbr{\alpha^{\star}} \gtrsim \frac12$ or $\epsilon^{-\frac1p} \nbr{z} \gtrsim \frac{1}{2}$.
Therefore, by Eq.~\eqref{eqn:w_n_V_n_decomp},
\[
\nbr{w_n}^2_{V_{n-1}^{-1}}  = \epsilon^{-\frac{2}{p}} \nbr{z}^2 +  \nbr{\alpha^{\star}}^2
\gtrsim \frac{1}{4}.
\]

\vspace{15pt}
\noindent To complete the proof, we justify the above inequalities in Eq.~\eqref{eqn:linear-task-eluder-key}.

\begin{enumerate}[label=(\alph*),topsep=0pt,itemsep=0pt,leftmargin=*]
    \item By Eq.~\eqref{eqn:task-eluder-linear-assumption} and $\epsilon$-independence, 
    for any $f'_n$ with corresponding vector $w'_n$, we have
    \[
    \nbr{B'w'_n - Bw_n}^p \gtrsim \EE_{(x,y) \sim P_{f_n \circ h}} \sbr{\ell((f'_n \circ h')(x),y) - \ell((f_n \circ h)(x), y)} > \frac{\epsilon}{2};
    \]
    that is,
    \[
    \min_{w'_n:  \nbr{w'_n} \in [\underline{b}, \overline{b}]} \nbr{B'w'_n - Bw_n} \gtrsim \epsilon^{\frac{1}{p}}.
    \]
    It then follows from Lemma~\ref{lem:subspace-distance-constrained} that
    \[
    \nbr{P_{B'}^\perp Bw_n} \gtrsim \min_{w'_n: \nbr{w'_n} \in [\underline{b}, \overline{b}]} \nbr{B'w'_n - Bw_n} \gtrsim \epsilon^{\frac{1}{p}}.
    \]
    
    \item uses the triangle inequality.

    \item By H\"{o}lder's inequality and the fact that $\nbr{P_{B'}^\perp Bw_i} = \min_{w' \in \RR^k} \nbr{B'w' - Bw_i}$, we have
    \begin{align*}
        \nbr{P_{B'}^\perp BW_{n-1} \alpha^{\star}}
        & \le \sum_{i=1}^{n-1} \abr{\alpha^{\star}_i} \nbr{P_{B'}^\perp Bw_i} \\
        & \le \rbr{\sum_{i=1}^{n-1} \nbr{P_{B'}^\perp Bw_i}^p}^{1/p} \rbr{\sum_{i=1}^{n-1} \abr{\alpha^{\star}}^q}^{1/q} \\
        & \le \rbr{\sum_{i=1}^{n-1} \nbr{B'w'_i - Bw_i}^p}^{1/p}  \nbr{\alpha^{\star}}_q,
    \end{align*}
    where $q \in [2, \infty)$ satisfies $\frac1p + \frac1q = 1$.
    Since $\nbr{\alpha^{\star}}_q \le \nbr{\alpha^{\star}}_2$, it suffices to show that $\sum_{i=1}^{n-1} \nbr{B'w'_i - Bw_i}^p \lesssim \epsilon$. By Eq.~\eqref{eqn:task-eluder-linear-assumption} and $\epsilon$-independence,
    \begin{align*}
        \sum_{i=1}^{n-1} \nbr{B'w'_i - Bw_i}^p
        \lesssim \sum_{i=1}^{n-1} \EE_{(x,y) \sim P_{f_i \circ h}} \sbr{\ell \rbr{(f'_i \circ h')(x), y} - \ell \rbr{(f_i \circ h)(x), y}} \le \epsilon. \quad 
    \end{align*}
\end{enumerate}
\end{proof}

\subsection{Noisy linear regression}
We first revisit the setting and restate Proposition~\ref{prop:example-linear-regression}.
\paragraph{Setting.} Recall that $\Xcal = \cbr{x \in \RR^d: \nbr{x} \le 1}$ and $\Ycal = [-1,1]$. We consider $\Hcal = \{x \mapsto B^\top x: B \in \RR^{d \times k}, B^\top B = I\}$ and $\Fcal^{\mathrm{lin}} = \cbr{z \mapsto z^\top w: w \in \RR^k, \nbr{w} \le \frac{1}{2}}$. 
The probabilistic model $\PP = \{P_{f \circ h}: h \in \Hcal, f \in \Fcal^{\mathrm{lin}}\}$ is defined as follows: $P_X$ satisfies $I \precsim \EE_{x \sim P_X} [xx^\top] \precsim I$. For each $h$ and $f$, given an input $x \sim P_X$, $y = (f \circ h)(x) + \eta$, where $\eta$ is independent noise from a distribution that has zero mean, support $[-\frac12, \frac12]$, and variance $\kappa$. The noise distribution is common to all $f \circ h$'s. 
There exist $h^*$ and $f^*_1, \ldots, f^*_T$ such that $\Pcal_t = P_{f_t^* \circ h^*}$, and let $B^*$ and $w_1^*, \ldots, w_T^*$ be the matrix and vectors associated with these functions. Let $\ell(y', y) = \frac{1}{4}(y' - y)^2$ be the loss function.

\begin{restate}{Proposition}{prop:example-linear-regression}
Let $\epsilon \in (0,1)$. We have
\[
\log \Ccal(\Fcal^{\mathrm{lin}}_{\ell}, \epsilon) \le \Ocal \Big(k \log\frac{1}{\epsilon}\Big), \ \log \Ccal_{\Fcal^{\mathrm{lin}}_{\ell}}(\Hcal, \epsilon) \le \Ocal \Big(dk \log \frac{1}{\epsilon} \Big), \
\dim_{\PP, \ell}(\Hcal, \Fcal^{\mathrm{lin}}, \epsilon) \le \Ocal \big(k \log \frac{1}{\epsilon}\big).
\]
\end{restate}

\begin{proof}[Proof of Proposition~\ref{prop:example-linear-regression}]
We prove the three statements separately.
\begin{enumerate}[leftmargin=*]
    \item \textbf{(Capacity of $\Fcal_{\ell}^{\mathrm{lin}}$)}
    Note that each representation $h$ maps $\Xcal$ to $\Zcal := \cbr{z \in \RR^k: \nbr{z} \le 1}$, and each prediction layer $f \in \Fcal^{\mathrm{lin}}$ maps $\Zcal$ to $[-\frac12,\frac12]$. Recall the definitions from Section~\ref{sec:additional-background}. We have
    \[
    \Ccal(\Fcal^{\mathrm{lin}}_\ell, \epsilon) = \sup_Q N(\Fcal^{\mathrm{lin}}, \epsilon, d_Q),
    \]
    where $d_Q(f_\ell, f'_\ell) = \int_{\Zcal \times \Ycal} \abr{\ell(f(z), y) - \ell(f'(z), y)} dQ(z,y)$ for any measure $Q$ on $\Zcal \times \Ycal$.

    We follow the technique from \citep{haussler1992decision,baxter2000model}. For any $Q$,
    \begin{align*}
        d_Q(f_\ell, f'_\ell) 
        & = \frac{1}{4} \int_{\Zcal \times \Ycal} \abr{(f(z) - y)^2 - (f'(z) - y)^2} dQ(z,y) \\
        & = \frac{1}{4} \int_{\Zcal \times \Ycal} \abr{(f(z) - f'(z))(f(z) + f'(z) - 2y)} dQ(z,y) \\
        & \le \int_{\Zcal} \abr{(f(z) - f'(z))} dQ_Z(z) =: L^1(Q),
    \end{align*}
    where $Q_Z$ is the marginal distribution derived from $Q$ and the inequality uses the observation that $\abr{f(z) + f'(z) - 2y} \le 3$. It follows that
    \[
    \sup_Q N(\Fcal^{\mathrm{lin}}_\ell, \epsilon, d_Q) \le \sup_Q N(\Fcal^{\mathrm{lin}}, \epsilon, L^1(Q)) \le \rbr{\frac{2e}{\epsilon}}^{2k},
    \]
    where the second inequality uses \citep[Theorem 11]{haussler1992decision}.
    We then have $\log \Ccal(\Fcal^{\mathrm{lin}}_\ell, \epsilon) \le \Ocal \rbr{k \log \frac{1}{\epsilon}}$.
    
    \item \textbf{(Capacity of $\Hcal$)} For any measure $P$ on $\Xcal \times \Ycal$, recall that
    \[
    d_{P, \Fcal^{\mathrm{lin}}_\ell}(h, h') = \int_{\Xcal \times \Ycal} \sup_{f_\ell \in \Fcal^{\mathrm{lin}}_\ell} \abr{f_\ell(h(x)),y) - f_\ell(h'(x)), y)} dP(x,y).
    \]
    For any $f \in \Fcal^{\mathrm{lin}}$, let $w$ be the vector associated with $f$. We have
    \begin{align*}
        \abr{f_\ell(h(x)),y) - f_\ell(h'(x)), y)}  
        & = \frac14 \abr{(f(h(x)) - y)^2 - (f(h'(x)) - y)^2} \\
        & = \frac14 \abr{(f(h(x)) - f(h'(x)))(f(h(x)) + f(h'(x)) - 2y)} \\
        & \le \frac34 \abr{\langle w, h(x) - h'(x) \rangle} \\
        & \le \nbr{h(x) - h'(x)}.
    \end{align*}
    Therefore,
    \[
    d_{P, \Fcal^{\mathrm{lin}}_\ell}(h, h') \le \int_{\Xcal} \nbr{h(x) - h'(x)} dP_X(x) =: L^1(P),
    \]
    where $P_X$ is the marginal distribution derived from $P$. By \citep[Theorem 11 therein]{haussler1992decision}, we then have
    \[
    \Ccal_{\Fcal^{\mathrm{lin}}_\ell}(\Hcal, \epsilon) := \sup_P N(\Hcal, \epsilon, d_{P, \Fcal^{\mathrm{lin}}_\ell}) \le \sup_P N(\Hcal, \epsilon,  L^1(P)) \le \rbr{\frac{2e}{\epsilon}}^{2dk},
    \]
    which completes the proof.

    \item \textbf{(Task-eluder dimension)} Observe that for any $h, h' \in \Hcal$ and $f, f' \in \Fcal^{\mathrm{lin}}$ with corresponding parameters $B, B'$ and  $w, w'$,
    \begin{align*}
        & \hspace{15pt} \EE_{P_{f \circ h}} \sbr{\ell((f' \circ h')(x), y) - \ell((f \circ h)(x), y)} \\
        & \stackrel{\mathrm{(a)}}{=} \EE_{x, \eta} \sbr{\rbr{x^\top B'w' - x^\top Bw  - \eta}^2 - \eta^2} \\
        & \stackrel{\mathrm{(b)}}{=} \rbr{B'w' - Bw} \EE_x [xx^\top] \rbr{B'w' - Bw},
    \end{align*}
    where (a) follows because under $P_{f \circ h}$, $y = (f \circ h)(x) + \eta = x^\top Bw + \eta$, and (b) follows because $\eta$ is independent of $x$ and $\EE[\eta] = 0$. Since $I \precsim \EE_{x \sim P_X} [xx^\top] \precsim I$,
    \[
    \nbr{B'w' - Bw}^2 \lesssim  \EE_{P_{f \circ h}} \sbr{\ell((f' \circ h')(x), y) - \ell((f \circ h)(x), y)} \lesssim \nbr{B'w' - Bw}^2.
    \]
    Applying Lemma~\ref{lem:task-eluder-linear} with $g(v) = v$ and $p = 2$, we have
    \[
    \dim_{\PP, \ell}(\Hcal, \Fcal^{\mathrm{lin}}, \epsilon) \lesssim k \log \frac{1}{\epsilon}. 
    \]
\end{enumerate}
\end{proof}

\subsection{Classification with logistic regression}
We first restate the setting and Proposition~\ref{prop:example-logistic}.
\paragraph{Setting.} Recall that $\Xcal = \cbr{x \in \RR^d: \nbr{x} \le 1}$ and $\Ycal = \cbr{0,1}$. We consider $\Hcal = \{x \mapsto B^\top x: B \in \RR^{d \times k}, B^\top B = I\}$ and $\Fcal^{\mathrm{log}} = \cbr{z \mapsto \sigma(z^\top w): w \in \RR^k, \nbr{w} \le \frac{1}{4}}$, where $\sigma(v) = 1/\rbr{1 + e^{-v}}$ is the logistic sigmoid function. The data distributions in $\PP = \cbr{P_{f \circ h}: h \in \Hcal, f \in \Fcal^{\mathrm{log}}}$ are defined as follows: $P_X$ satisfies $I \precsim \EE_{x \sim P_X} [xx^\top] \precsim I$, and for each $h$ and $f$, $\Pr(y = 1 | x; f \circ h) = \sigma((f \circ h)(x))$ and $\Pr(y = 0 | x; f \circ h) = 1 -\sigma((f \circ h)(x))$. There exist $h^*$ and $f^*_1, \ldots, f^*_T$ such that $\Pcal_t = P_{f_t^* \circ h^*}$, and let $B^*$ and $w_1^*, \ldots, w_T^*$ be the associated parameters. Let $\ell(y', y) = -y \log y' - (1-y) \log(1 - y')$ be the loss function.
\begin{restate}{Proposition}{prop:example-logistic}
Let $\epsilon \in (0,1)$. We have
\[
\log \Ccal(\Fcal^{\mathrm{log}}_\ell, \epsilon) \le \Ocal \Big(k \log\frac{1}{\epsilon}\Big), \ \log \Ccal_{\Fcal^{\mathrm{log}}_\ell}(\Hcal, \epsilon) \le \Ocal \Big(dk \log \frac{1}{\epsilon} \Big), \
\dim_{\PP, \ell}(\Hcal, \Fcal^{\mathrm{log}}, \epsilon) \le \Ocal \big(k \log \frac{1}{\epsilon}\big).
\]
\end{restate}

\begin{proof}[Proof of Proposition~\ref{prop:example-logistic}]
Again, we prove the three statements separately.
\begin{enumerate}[leftmargin=*]
    \item \textbf{(Capacity of $\Fcal^{\mathrm{log}}$)} As in the proof of Proposition~\ref{prop:example-linear-regression}, we use the technique from \citep{haussler1992decision,baxter2000model}. First observe that each $h$ is a mapping from $\Xcal$ to $\Zcal := \cbr{z \in \RR^k: \nbr{z} \le 1}$, and each $f \in \Fcal^{\mathrm{log}}$ is a mapping from $\Zcal$ to $[\sigma(-\frac14), \sigma(\frac14)]$. 
    For any probability measure $Q$ on $\Zcal \times \Ycal$, we have
    \begin{align*}
        d_Q(f_\ell, f'_\ell) 
        & = \int_{\Zcal \times \Ycal} \abr{\ell(f(z), y) - \ell(f'(z), y)} dQ(z,y) \\
        & \le 3 \underbrace{\int_{\Zcal} \abr{f(z) - f'(z)} dQ_Z(z)}_{=: L^1(Q)},
    \end{align*}
    where $Q_Z$ is the marginal distribution derived from $Q$, and the inequality follows because $\ell$ is $3$-Lipschitz continuous with respect to its first argument over the domain. It follows that
    \[
    N(\Fcal^{\mathrm{log}}_\ell, \epsilon, d_Q) \le N(\Fcal^{\mathrm{log}}, \frac{\epsilon}{3}, L^1(Q)),
    \]
    and therefore,
    \[
    \Ccal(\Fcal^{\mathrm{log}}_\ell, \epsilon) := \sup_Q N(\Fcal^{\mathrm{log}}_\ell, \epsilon, d_Q) \le \sup_Q N(\Fcal^{\mathrm{log}}, \frac{\epsilon}{3}, L^1(Q)) \le \rbr{\frac{6e}{\epsilon}}^{2k},
    \]
    where the last inequality uses \citep[Theorem 11]{haussler1992decision}.
  
    \item \textbf{(Capacity of $\Hcal$)}
    For any probability measure $P$ on $\Xcal \times \Ycal$, consider
    \begin{align*}
        d_{P, \Fcal^{\mathrm{log}}_\ell} (h, h') 
        & = \int_{\Xcal \times \Ycal} \sup_{f_\ell \in \Fcal^{\mathrm{log}}_\ell} \abr{\ell((f \circ h)(x), y) - \ell((f \circ h')(x), y)} dP(x,y). 
    \end{align*}
    For any $f \in \Fcal^{\mathrm{log}}$, $h, h' \in \Hcal$, $x \in \Xcal$, and $y \in \Ycal$, since $\ell$ is $3$-Lipschitz continuous with respect to the first argument over the domain,
    \begin{align*}
        \abr{\ell((f \circ h)(x), y) - \ell((f \circ h')(x), y)}
        \le 3 \abr{(f \circ h)(x) - (f \circ h')(x)}.
    \end{align*}
    Let $w$, $B$, and $B'$ denote the parameters associated with $f$, $h$, and $h'$, respectively. It follows that
    \begin{align*}
        \abr{(f \circ h)(x) - (f \circ h')(x)} 
        & = \abr{\sigma\rbr{w^\top h(x)} - \sigma \rbr{w^\top h'(x)}} \\
        & \le \frac14 \abr{w^\top (h(x) - h'(x))} \\
        & \le \frac{1}{16} \nbr{h(x) - h'(x)},
    \end{align*}
    where the second inequality uses the observation that $\sigma(\cdot)$ is $\frac{1}{4}$-Lipschitz continuous, and the last inequality follows because $\|w\| \le \frac{1}{4}$.
    Therefore,
    \[
    d_{P, \Fcal^{\mathrm{log}}_\ell} (h, h') \le \int_{\Xcal} \nbr{h(x) - h'(x)} dP_X(x) =: L^1(P),
    \]
    where $P_X$ is the marginal distribution derived from $P$. It then follows from \citep[Theorem 11]{haussler1992decision} that
    \[
    \Ccal_{\Fcal^{\mathrm{log}}_\ell}(\Hcal, \epsilon) := \sup_P N(\Hcal, \epsilon, d_{P, \Fcal^{\mathrm{log}}_\ell}) \le \sup_P N(\Hcal, \epsilon, L^1(P)) \le \rbr{\frac{2e}{\epsilon}}^{2dk},
    \]
    which completes the proof.

    \item \textbf{(Task-eluder dimension)} Fix any $h \in \Hcal$ and $f \in \Fcal^{\mathrm{log}}$. Let $B$ and $w$ denote the associated matrix and vector, respectively.
    
    For any $\theta \in \RR^d$ such that $\|\theta\| \le \frac14$, consider
    \[
    L(\theta) := \EE_{(x,y) \sim P_{f \circ h}} [\ell(\sigma(x^\top \theta), y)].
    \]
    It is easy to verify that 
    \[
    \nabla^2_\theta \ L(\theta) = \EE_{x \sim P_X} [\sigma'(x^\top \theta) xx^\top].
    \]
    For the domain $\abr{x^\top \theta} \le \frac14$, $1 \lesssim \sigma'(x^\top \theta) \lesssim 1$. Since $I \precsim \EE_{x \sim P_X} [xx^\top] \precsim I$,
    \[
    I \precsim  \nabla^2_\theta L(\theta) \precsim I;
    \]
    that is, $L(\theta)$ is {\em locally} strongly convex and smooth. Therefore, for any $h'$ and $f'$ with corresponding parameters $B'$ and $w'$,
    \[
    \nbr{B'w' - Bw}^2 \lesssim \underbrace{L(B'w') - L(Bw)}_{= \EE_{P_{f \circ h}} [\ell((f' \circ h')(x), y) - \ell((f \circ h)(x), y)]} \lesssim \nbr{B'w' - Bw}^2.
    \]

    Applying Lemma~\ref{lem:task-eluder-linear} with $g = \sigma$ and $p = 2$, we have
    \[
    \dim_{\PP, \ell} \rbr{\Hcal, \Fcal^{\mathrm{log}}, \epsilon} \lesssim k \log \frac{1}{\epsilon}. 
    \]
\end{enumerate}
\end{proof}

\subsection{Classification with random classification noise and the \texorpdfstring{$0$-$1$}{0-1} loss}

We now discuss how our results may also be applied to binary classification with random classification noise under the $0$-$1$ loss.

\paragraph{Setting.} Let $\Xcal = \RR^d$ and $\Ycal = \cbr{- 1,1}$. We again consider low-dimensional linear representations, $\Hcal = \cbr{x \mapsto B^\top x, B \in \RR^{d \times k}, B^\top B = I}$. 
Let $\Fcal^{\mathrm{cls}} = \{z \mapsto \sign \rbr{\langle w, z \rangle}: w \in \RR^k, \nbr{w} = 1\}$ be a class of linear threshold functions. Consider the probabilistic model $\PP$: let $P_X$ be isotropic log-concave (e.g., normal distribution and uniform distribution). 
We consider random classification noise \citep[e.g.,][]{kearns1994introduction}. For any $h$ and $f$, given an input $x \sim P_X$, $y =\sign((f \circ h)(x))$ with probability $ 1 - \eta$, and $y = -\sign((f\circ h)(x))$ with probability $\eta$, where $\eta \in [0, \frac{4}{10}]$ is the noise rate. There exist $B^*$ and $w_1^*, \ldots, w_T^*$ such that each task $t$ is well specified by $B^*$ and $w_t^*$. Let $\ell(y',y) = \ind \cbr{y' \neq y}$ denote the $0$-$1$ loss.
\begin{proposition}
\label{prop:example-classification}
Let $\epsilon \in (0,1)$. We have
\[
\dim_{\PP, \ell}(\Hcal, \Fcal^{\mathrm{cls}}, \epsilon) \le \Ocal \big(k \log \frac{1}{\epsilon}\big).
\]
\end{proposition}

\begin{proof}[Proof of Proposition~\ref{prop:example-classification}]
Since $P_X$ is isotropic and log-concave in $\Xcal$, it follows from \citep[Lemma 1 therein]{balcan2015efficient} that for any two unit vectors $u$ and $v$,
\[
\nbr{u - v} \lesssim \Pr_{x \sim P} \rbr{\sign\rbr{u^\top x} \neq \sign\rbr{v^\top x}} \lesssim \nbr{u - v},
\]
where we also use the fact that the Euclidean distance between two unit vectors is equivalent to the angle between them up to constant factors.

In addition, under random classification noise, for any $h, h'$ and $f, f'$,
\begin{align*}
    & (1 - 2\eta) \Pr((f' \circ h')(x) \neq (f \circ h)(x)) \le \EE_{P_{f \circ h}} \sbr{\ell((f' \circ h')(x), y) -\ell((f \circ h)(x), y)} \\
    & \hspace{250pt} \le \Pr((f' \circ h')(x) \neq (f \circ h)(x)).
\end{align*}
See \citep[Section 5.1]{balcan2020noise} for a reference. 

For any $h,h' \in \Hcal$ and $f, f' \in \Fcal^{\mathrm{cls}}$, let $B, B'$ and $w, w'$ denote the corresponding parameters, respectively. Since $\eta \in [0, \frac{4}{10}]$, it follows that
\[
\nbr{B'w' - Bw} \lesssim \EE_{P_{f \circ h}} \sbr{\ell((f' \circ h')(x), y) -\ell((f \circ h)(x), y)} \lesssim \nbr{B'w' - Bw}.
\]
The proof is completed by applying Lemma~\ref{lem:task-eluder-linear} with $g(v) = \sign(v)$, $\underline{b} = \overline{b} = 1$, and $p = 1$. 
\end{proof}

\begin{remark}
Based on Proposition~\ref{prop:example-classification}, one may derive a sample complexity guarantee for the classification example above.
We note that Theorem~\ref{thm:main} leverages sample complexity guarantees for multi-ERM based on covering numbers \citep{baxter2000model}, which may no longer be suitable under the discrete $0$-$1$ loss.
Nevertheless, recent work by \citet{aliakbarpour2024metalearning} establishes multi-task ERM guarantees based on the VC dimension, offering a path towards proving a result analogous to Theorem~\ref{thm:main}. We defer a detailed exploration of this direction to future work.
\end{remark}

\subsection{Auxiliary lemmas} 
\label{app:auxiliary-lemmas}
We now present the lemmas used above in the proof of Proposition~\ref{prop:example-classification}.

\begin{lemma}[\citealp{abbasi2011improved}, Lemma 11]
\label{lem:elliptical-potential}
Let $\cbr{x_t}_t$ be a sequence of vectors in $\RR^d$ such that $\nbr{x_t} \le L$ for all $t$. Let $V_t = \lambda I + \sum_{s=1}^{t} x_s x_s^\top$. Then,
\[
\sum_{t=1}^T \min \cbr{1, \nbr{x_t}_{{V}_{t-1}^{-1}}^2} \le 2d \log \rbr{1 + \frac{L^2T}{d\lambda}}.
\]
\end{lemma}

\begin{lemma}[\citealp{lattimore2020bandit}, Exercise 19.3]
\label{lem:elliptical-potential-bound-t}
Let $\epsilon \in (0,1)$. Suppose $z \le c k \log (1 + \frac{z}{k\epsilon^2})$ for some constant $c > 1$. Then, 
\[
z \le 12c \cdot k \log \frac{1}{\epsilon}.
\]
\end{lemma}
\begin{proof}
This proof is due to \citep{lattimore2020bandit} for a closely related result. We provide it here with elaborated details for clarity and completeness.

Let $x = ck > 0$ and $y = \frac{1}{k\epsilon^2} > 0$. 
In addition, let $g(z) = z - x \log(1 + yz)$, and $z_0 = 3x \log(1 + xy)$.
It suffices to show that $g(z) \le 0$ implies $z \le z_0$. The rest follows by algebra: $\log\rbr{1 + \frac{c}{\epsilon^2}} \le 4\log \frac{1}{\epsilon}$ for $\epsilon \in (0,1)$ and $c > 1$.

To this end, we show that: (1) $g(z_0) \ge 0$ and (2) for $z > z_0$, $g(z)$ is increasing.
For (1), observe that
\begin{align*}
    x\log(1 + y z_0) & = x\log(1 + y \rbr{3x\log(1+xy)}) \\
    & \stackrel{\mathrm{(a)}}{\le} x\log(1 + 3x^2y^2) \\
    & \stackrel{\mathrm{(b)}}{\le} x \log(1+xy)^3 = z_0,
\end{align*}
where (a) follows because $\log (1 + a) \le a$ for $a > -1$, and (b) follows because $(1+a)^3 = 1 + 3a + 3a^2 + a^3 \ge 1 + 3a^2$ for $a > 0$.
For (2), first note that
\[
\frac{dg(z)}{dz} = 1 - \frac{xy}{1 + yz}.
\]
Therefore, $g(z)$ is strictly increasing when $z > x - \frac{1}{y}$, and so it suffices to show that $z_0 > x - \frac{1}{y}$. To do so, let $b = xy$. We have 
\begin{align*}
    3b \log(1+b) \stackrel{\mathrm{(c)}}{\ge} \frac{3b^2}{1 + b} \stackrel{\mathrm{(d)}}{>} b -1,
\end{align*}
where (c) uses the fact that $\log(1 + a) \ge \frac{a}{1+a}$ for $a > -1$, and (d) follows because $3b^2 - (1+b)(b-1) = 2b^2 + 1 > 0$. It then follows that
\[
z_0 = 3x \log(1 + xy) > x - \frac{1}{y},
\]
which completes the proof. 
\end{proof}

\begin{lemma}
    \label{lem:ridge-woodbury-identity}
    For any $x \in \RR^d$, $U \in \RR^{d \times n}$, and $\lambda > 0$,
    \begin{align*}
        x^\top \rbr{UU^\top + \lambda I}^{-1} x = \min_{z \in \RR^{n}} \frac{1}{\lambda} \nbr{x - U z}_2^2 + \nbr{z}_2^2.
    \end{align*}
\end{lemma}

\begin{proof}
Let $g(z) :=  \frac{1}{\lambda} \nbr{x - U z}_2^2 + \nbr{z}_2^2$. By a little algebra, we have
\begin{align*}
g(z) = \frac{1}{\lambda} z^\top \rbr{\lambda I + U^\top U} z - \frac{2}{\lambda} x^\top U z + \frac{1}{\lambda} x^\top x
\end{align*}
It is easy to verify that $g(z)$ is convex and minimized at $z_* = \rbr{\lambda I + U^\top U}^{-1} U^\top x$. 

It then follows that
\begin{align*}
    g(z_*) 
    & = \frac{1}{\lambda} x^\top U \rbr{\lambda I + U^\top U}^{-1} U^\top x - \frac{2}{\lambda} x^\top U \rbr{\lambda I + U^\top U}^{-1} U^\top x + \frac{1}{\lambda} x^\top x \\
    & = \frac{1}{\lambda} x^\top x - \frac{1}{\lambda} x^\top U \rbr{\lambda I + U^\top U}^{-1} U^\top x \\
    & \stackrel{\mathrm{(a)}}{=} x^\top \rbr{\lambda I + U U^\top}^{-1} x,
\end{align*}
where (a) uses the Woodbury matrix identity:
\[
\rbr{\lambda I + UU^\top}^{-1} = \frac{1}{\lambda} I - \frac{1}{\lambda} U \rbr{\lambda I + U^\top U}^{-1} U^\top. 
\]
\end{proof}

\begin{lemma}
\label{lem:subspace-distance-constrained}
Let $B \in \RR^{d \times k}$ be an orthonormal basis of a $k$-dimensional subspace of $\RR^d$. Let $u$ be a vector in $\RR^d$ such that $0 \le \underline{b} \le \nbr{u} \le \overline{b} \le 1$. Then,
\[
\min_{w \in \RR^k: \nbr{w} \in [\underline{b}, \overline{b}]} \nbr{Bw - u} \le 2 \nbr{P_{B}^\perp u}.
\]
\end{lemma}
\begin{proof}
    Let $\beta = \nbr{u} \in [\underline{b}, \overline{b}]$, and let $z = B^\top u$. It follows that $\nbr{z} \in [0, \beta]$. We consider three cases:
    \begin{enumerate}[leftmargin=*]
        \item $\nbr{z} \ge \underline{b}$. 
        In this case we have $\nbr{z} \in [\underline{b}, \overline{b}]$, and so
        \begin{align*}
            \min_{w \in \RR^k: \nbr{w} \in [\underline{b}, \overline{b}]} \nbr{Bw - u} \le \nbr{Bz - u} = \nbr{BB^\top u - u} = \nbr{P_B^\perp u}.
        \end{align*}

        \item $\nbr{z} = 0$. It follows that
        \[
        \nbr{P_B^\perp u} = \nbr{u} = \beta,
        \]
        and for any $w' \in \RR^k$ such that $\nbr{w'} = \underline{b}$, by the triangle inequality,
        \begin{align*}
            \min_{w \in \RR^k: \nbr{w} \in [\underline{b}, \overline{b}]} \nbr{Bw - u} \le \nbr{Bw' - u} \le \nbr{Bw'} + \nbr{u} \le \underline{b} + \beta \le 2 \nbr{P_B^\perp u}.
        \end{align*}

        \item $0 < \nbr{z} < \underline{b}$. Consider $\hat{z} = \underline{b} \frac{z}{\nbr{z}}$. We have $\nbr{\hat{z}} = \underline{b}$. Then,
        \begin{align*}
            \hspace{-20pt} \min_{w \in \RR^k: \nbr{w} \in [\underline{b}, \overline{b}]} \nbr{Bw - u} 
            & \le \nbr{B\hat{z} - u} \\
            & \stackrel{\mathrm{(a)}}{\le} \nbr{B\hat{z} - Bz} + \nbr{Bz - u} \\
            & \stackrel{\mathrm{(b)}}{\le} \beta - \nbr{z} + \nbr{P_B^\perp u} \\
            & \stackrel{\mathrm{(c)}}{\le} 2\nbr{P_B^\perp u},
        \end{align*}
        where (a) uses the triangle inequality;
        (b) follows because 
        \[
        \nbr{B\hat{z} - Bz} = \nbr{\hat{z} - z} = \nbr{\rbr{\frac{\underline{b}}{\nbr{z}} - 1} z} = \rbr{\frac{\underline{b}}{\nbr{z}} - 1} \nbr{z} = \underline{b} - \nbr{z} \le \beta - \nbr{z};
        \]
        and (c) again uses the triangle inequality:
        \[
        \beta = \nbr{u} \le \nbr{P_B u} + \nbr{P_B^\perp u}
        = 
        \nbr{z} + \nbr{P_B^\perp u}.
        \]
    \end{enumerate}     
\end{proof}

\section{Supplementary material for Section~\ref{sec:empirical-validation}}
\label{app:implementation-details}

\subsection{Implementation details}

\noindent All implementations were developed in PyTorch. Most experiments were conducted on machines equipped with NVIDIA GeForce RTX 4090 GPUs. 

\paragraph{Synthetic linear experiments.} For data generation, in each trial, $B^*$ was obtained from the QR decomposition of a random $d \times k$ matrix with standard normal entries, and $w^*_t$’s were drawn from the uniform distribution over the origin-centered sphere with radius $\beta$. For both multi-task ERM and few-shot property tests, we used batch gradient descent with the Adam optimizer \citep{kingma2014adam} (learning rate $10^{-3}$), for a maximum of $10^4$ epochs with early stopping after $20$ epochs without improvement. The Bayes-optimal risks were estimated via Monte-Carlo simulations from $10^6$ samples.

\paragraph{MNIST experiments.} Both multi-task ERM and few-shot property tests used batch gradient descent with the Adam optimizer (learning rate $10^{-3}$), for up to $2000$ epochs. For multi-task ERM, early stopping was applied after $10$ epochs without improvement.

\paragraph{CIFAR-10 experiments.}
We consider representations given by a modified ResNet-18 architecture \citep{he2016deep}. The initial convolutional layer is replaced by a $3 \times 3$ kernel with stride $1$, with the remaining convolutional blocks kept unchanged. After the convolutional blocks, we apply global average pooling, followed by a fully-connected linear layer, batch normalization, and ReLU activation. The resulting representation is of dimension $k = 256$. All representation networks were trained from scratch without any pretrained weights.

For multi-task ERM, we used mini-batch stochastic gradient descent with the Adam optimizer (learning rate $10^{-3}$), a batch size of $256$, and trained for up to $300$ epochs with early stopping after $20$ epochs without improvement. For few-shot property tests, we used batch gradient descent with the Adam optimizer (learning rate $10^{-2}$) for a maximum of $200$ epochs.

\end{bibunit}

\end{document}